\RequirePackage{silence}
\documentclass[onefignum,onetabnum]{siamart220329}

\usepackage[round]{natbib}

\usepackage{amsmath,amsfonts,amssymb}
\usepackage{mathtools}
\usepackage{bm}
\usepackage{bbm}
\allowdisplaybreaks

\usepackage{bookmark}

\usepackage{graphicx}
\graphicspath{{./}}
\usepackage{multirow}
\usepackage{booktabs}
\usepackage{makecell}
\usepackage{tikz}
\usetikzlibrary{positioning}

\newsiamremark{remark}{Remark}
\newsiamthm{assumption}{Assumption}

\makeatletter
\AtBeginDocument{%
  \let\cref@override@label@type\@gobble
}
\makeatother

\crefname{theorem}{Theorem}{Theorems}
\Crefname{theorem}{Theorem}{Theorems}
\crefname{lemma}{Lemma}{Lemmas}
\Crefname{lemma}{Lemma}{Lemmas}
\crefname{corollary}{Corollary}{Corollaries}
\Crefname{corollary}{Corollary}{Corollaries}
\crefname{proposition}{Proposition}{Propositions}
\Crefname{proposition}{Proposition}{Propositions}
\crefname{definition}{Definition}{Definitions}
\Crefname{definition}{Definition}{Definitions}
\crefname{remark}{Remark}{Remarks}
\Crefname{remark}{Remark}{Remarks}
\crefname{assumption}{Assumption}{Assumptions}
\Crefname{assumption}{Assumption}{Assumptions}

\newcommand{\cz}{\mathcal{Z}}
\newcommand{\cs}{\mathcal{S}}
\newcommand{\cf}{\mathcal{F}}

\newcommand{\ch}{\mathcal{H}}

\newcommand{\be}{\mathbb{E}}
\newcommand{\hsic}{\mathrm{HSIC}}

\title{A Joint-Distribution Route to Fair Representations with Continuous Sensitive Attributes\thanks{\funding{This work was supported in part by NSF grant 2229876, the A.\ Russell Chandler III Professorship at Georgia Institute of Technology, an NIH-sponsored Georgia Clinical \& Translational Science Alliance, and the Georgia Department of Transportation.}}}

\headers{Joint-Distribution Route to Fair Representations}{Y.\ Ni and X.\ Huo}

\author{Yijin Ni\thanks{H.\ Milton Stewart School of Industrial and Systems Engineering, Georgia Institute of Technology, Atlanta, GA 30332-0205 USA (\email{yni64@gatech.edu}, \email{huo@gatech.edu}). ORCID iDs: \href{https://orcid.org/0000-0002-7209-6887}{0000-0002-7209-6887} (Ni) and \href{https://orcid.org/0000-0003-0101-1206}{0000-0003-0101-1206} (Huo).}\and
Xiaoming Huo\footnotemark[2]}

\begin{document}
\setlength{\emergencystretch}{3em}

\maketitle

\begin{abstract}
Fair representation learning with a continuous sensitive attribute $S$ requires a representation $Z$ that is statistically independent of $S$. Existing criteria, including generalized demographic parity, the expectation of integral probability metrics (EIPM), and mutual information, enforce this independence by averaging a per-value discrepancy between the conditional law $P_{Z\mid S=s}$ and the marginal $P_Z$ over the law of $S$. This approach requires a nonparametric surrogate for the conditional law at each sensitive value. We propose evaluating independence through a single joint discrepancy $d(P_{Z,S}, P_Z\otimes P_S)$ between the joint law and the product of its marginals. We establish a disintegration identity; on decomposable witness classes it equals the conditional-integral functional that EIPM and generalized demographic parity instantiate. By reaching the same target without the conditional law, this discrepancy can be estimated directly from samples via a dependence statistic rather than conditional smoothing. We take the Hilbert--Schmidt independence criterion (HSIC) as an instance of the joint discrepancy $d$ to investigate the statistical efficiency of replacing the conditional formulation. The HSIC estimator is a closed-form $O(n^2)$ statistic that converges at the $O(n^{-1/2})$ rate, in contrast to the nonparametric $O(n^{-2/5})$ rate of the conditional-route estimators. We prove this instance is equivalent to the conditional maximum mean discrepancy (MMD) integral up to an explicit spectral tail. The corresponding algorithmic implementation, i.e., FRHSIC, attains fairness--accuracy tradeoffs comparable to conditional-route baselines while reducing per-epoch training time.
\end{abstract}



%======================================================================
\section{Introduction}
\label{sec:introduction}
%======================================================================

Fair representation learning seeks a representation $Z$ that is statistically independent of a sensitive attribute $S\in\mathcal S$, i.e., $Z\perp S$, so that any downstream predictor built on $Z$ inherits demographic parity with respect to $S$~\citep{zemel2013learning,madras2018learning}. An encoder $h:\mathcal X\to\mathcal Z$ maps an input $X\in\mathcal X$ to the representation $Z=h(X)\in\mathcal Z$, and any prediction head acting on $Z$ is fair once $Z$ carries no information about $S$, which makes $Z\perp S$ the representation-level target. Writing $P_Z$ and $P_S$ for the marginal laws of $Z$ and $S$ and $P_{Z\mid S=s}$ for the conditional law of $Z$ given $S=s$, this target is the equality $P_{Z\mid S=s}=P_Z$ for $P_S$-almost every $s$. We study it when $S$ is continuous, as with age, income, or a risk score in domains such as hiring, lending, and criminal justice~\citep{calders2010three}, where independence must hold across a continuum of sensitive values rather than across finitely many groups.

For continuous $S$, existing criteria enforce this independence by approximating a localized conditional object indexed by $S=s$: the conditional law $P_{Z\mid S=s}$, or, for generalized demographic parity (GDP), the conditional mean $\mathbb E[\,\cdot\mid S=s]$ of the prediction head. They share one form, the $S$-average of a per-value discrepancy $d$ between the conditional law and the marginal,
\begin{equation}
    \label{eq:integralIPM}
    \mathcal I_d(Z;S):=\mathbb E_S\!\left[d(P_{Z\mid S},P_Z)\right],
\end{equation}
and differ only in $d$: GDP uses a first-moment difference~\citep{jiang2022generalized}, the expectation-of-IPM criterion (EIPM) of \citet{kong2025fair} uses an integral probability metric (IPM), realized by the kernel-weighted estimator FREM, and mutual-information objectives use the Kullback--Leibler divergence~\citep{cho2020fair}. Because a finite sample of size $n$ has no repeated observations at an exact value $S=s$, each criterion replaces exact conditioning by a localized or density-based surrogate, such as FREM's leave-one-out kernel-weighted conditional empirical measure or GDP's Nadaraya--Watson kernel smoothing on $S$~\citep{nadaraya1964estimating,watson1964smooth}. These smoothing-based estimators carry a localization bandwidth and converge at the one-dimensional nonparametric rate $O(n^{-2/5})$.

We propose to enforce independence through a single discrepancy between the joint law $P_{Z,S}$ of the pair $(Z,S)$ and the product of its marginals $P_Z\otimes P_S$, the product measure on $\mathcal Z\times\mathcal S$; Figure~\ref{fig:routes} contrasts this joint route with the conditional route it replaces. For a class $\mathcal F$ of bounded measurable functions on $\mathcal Z\times\mathcal S$, we measure this discrepancy by the joint-vs-product integral probability metric
\begin{equation*}
\mathrm{IPM}_{\mathcal F}\bigl(P_{Z,S},\,P_Z\otimes P_S\bigr)
:=\sup_{f\in\mathcal F}\bigl|\mathbb E_{P_{Z,S}}[f]-\mathbb E_{P_Z\otimes P_S}[f]\bigr|,
\end{equation*}
which vanishes if and only if $Z\perp S$ whenever $\mathcal F$ is rich enough to separate distinct laws (a characteristic class), and at this target it agrees with the conditional-integral criteria of~\eqref{eq:integralIPM}, namely GDP, EIPM, and mutual information (Proposition~\ref{prop:equivalence}). The averaging over sensitive values is then intrinsic to the joint law rather than an imposed weighting, and no per-value conditional object enters the definition.

The link to the conditional-integral route is not confined to this zero level: the joint discrepancy disintegrates exactly into the same averaged form. A disintegration identity (Theorem~\ref{thm:joint-equals-integral}) rewrites the joint-vs-product IPM, for any class $\mathcal F$, as the $P_S$-average of a conditional contrast with the supremum over $\mathcal F$ taken outside the average,
\begin{equation*}
\mathrm{IPM}_{\mathcal F}\bigl(P_{Z,S},\,P_Z\otimes P_S\bigr)
=\sup_{f\in\mathcal F}\left|\int_{\mathcal S}\!\int_{\mathcal Z} f(z,s)\,d\bigl(P_{Z\mid S=s}-P_Z\bigr)(z)\,dP_S(ds)\right|
\end{equation*}
and on the decomposable classes of Definition~\ref{def:decomposable-class}, generated by per-value witness functions, this average equals the conditional-integral functional $\mathcal I_d$ of~\eqref{eq:integralIPM} that GDP and EIPM instantiate (Corollary~\ref{cor:decomposable-class}). Mutual information is the classical precedent for this coincidence: as the Kullback--Leibler discrepancy it equals both the joint-vs-product divergence $\mathrm{KL}(P_{Z,S}\,\|\,P_Z\otimes P_S)$ and the $P_S$-average of the conditional divergences $\mathrm{KL}(P_{Z\mid S=s}\,\|\,P_Z)$ by the chain rule of relative entropy~\citep[Ch.~2]{coverthomas2006}, so it already measures independence from the joint law without constructing the conditional family $\{P_{Z\mid S=s}\}_s$.

Because the target is a single functional of the paired law, it is estimated directly from the joint sample $\{(Z_i,S_i)\}_{i=1}^n$ using only those paired observations and the witness class $\mathcal F$. How fast this direct estimate converges then depends on $\mathcal F$ rather than on a smoothing bandwidth.

We take the Hilbert--Schmidt independence criterion (HSIC) as the instance obtained by choosing $\mathcal F$ to be the unit ball of a product reproducing-kernel Hilbert space (RKHS). The empirical HSIC is a closed-form statistic, computable in $O(n^2)$ time from the Gram matrices of $\{Z_i\}_{i=1}^n$ and $\{S_i\}_{i=1}^n$, whose kernel bandwidths are fixed scale parameters rather than per-value localization for the conditional law. Being a non-degenerate $V$-statistic, it converges to its population value at the $O(n^{-1/2})$ rate, in contrast to the nonparametric $O(n^{-2/5})$ of the smoothing-based EIPM estimator (FREM); our synthetic study fits log--log slopes of $-0.46$ for HSIC and $-0.44$ for EIPM, close to $-1/2$ and $-2/5$ (\S\ref{sec:exp-synthetic}, Figure~\ref{fig:convergence}). This rate is not only pointwise: up to logarithmic factors it holds uniformly over encoder classes of controlled complexity, such as bounded linear encoders and fixed-architecture bounded multilayer perceptrons (Theorem~\ref{thm:uniform-hsic-frhsic}, building on \citet{ni2024uniform}), so the data-dependent encoder that minimizes the penalty still has small population dependence. We further prove that HSIC is equivalent to the $P_S$-averaged conditional maximum mean discrepancy (MMD) integral up to an explicit spectral tail of the sensitive-attribute kernel (Theorem~\ref{cor:hsic-mmd-equivalence}). As a minibatch regularizer, the resulting objective, which we call FRHSIC, trains about $36\times$ faster per epoch than FREM at $n=20{,}000$ (\S\ref{sec:runtime}); penalizing dependence with HSIC for fairness is itself established~\citep{perezsuay2017fair,li2022kernel}, and our contribution is the theory that follows.

\subsection{Summary of Contributions}\label{sec:contributions}

\paragraph{A joint-discrepancy formulation of the continuous-$S$ criteria}
We show that the existing continuous-$S$ fairness criteria arise, on decomposable witness classes, as the conditional reading of a single joint discrepancy. A disintegration identity (Theorem~\ref{thm:joint-equals-integral}) rewrites the joint-vs-product IPM as a $P_S$-averaged conditional contrast, and on decomposable witness classes (Corollary~\ref{cor:decomposable-class}) it coincides with the conditional-integral functional $\mathcal I_d$ of~\eqref{eq:integralIPM} that generalized demographic parity and EIPM instantiate, while mutual information is the Kullback--Leibler instance of the same $S$-averaged form.

\paragraph{A closed-form HSIC estimator with spectral and uniform control}
We give a closed-form $O(n^2)$ estimator of the joint discrepancy, the HSIC obtained from a product RKHS, that is controlled by the $P_S$-averaged conditional object without ever estimating it. We prove it is equivalent to the conditional MMD integral up to an explicit spectral tail of the sensitive-attribute kernel (Theorem~\ref{cor:hsic-mmd-equivalence}), bounds the GDP of an RKHS head (Corollary~\ref{cor:spectral_gdp_control}), and concentrates uniformly at the rate $n^{-1/2}$, up to logarithmic factors, under a controlled-complexity condition on the encoder class (Theorem~\ref{thm:uniform-hsic-frhsic}, building on \citet{ni2024uniform}).

\paragraph{A training algorithm for fair representations, FRHSIC}
We instantiate the estimator as a minibatch regularizer that attains fairness--accuracy tradeoffs comparable to conditional-route baselines while training about $36\times$ faster per epoch than FREM at $n=20{,}000$ (\S\ref{sec:runtime}). On synthetic data and five real datasets, FRHSIC attains comparable fairness--accuracy tradeoffs, shows estimator convergence consistent with the predicted $O(n^{-1/2})$ rate (\S\ref{sec:exp-synthetic}, Figure~\ref{fig:convergence}), and keeps fairness stable across fresh downstream heads.

\begin{figure}[!htbp]
    \centering
    \includegraphics[width=0.95\textwidth]{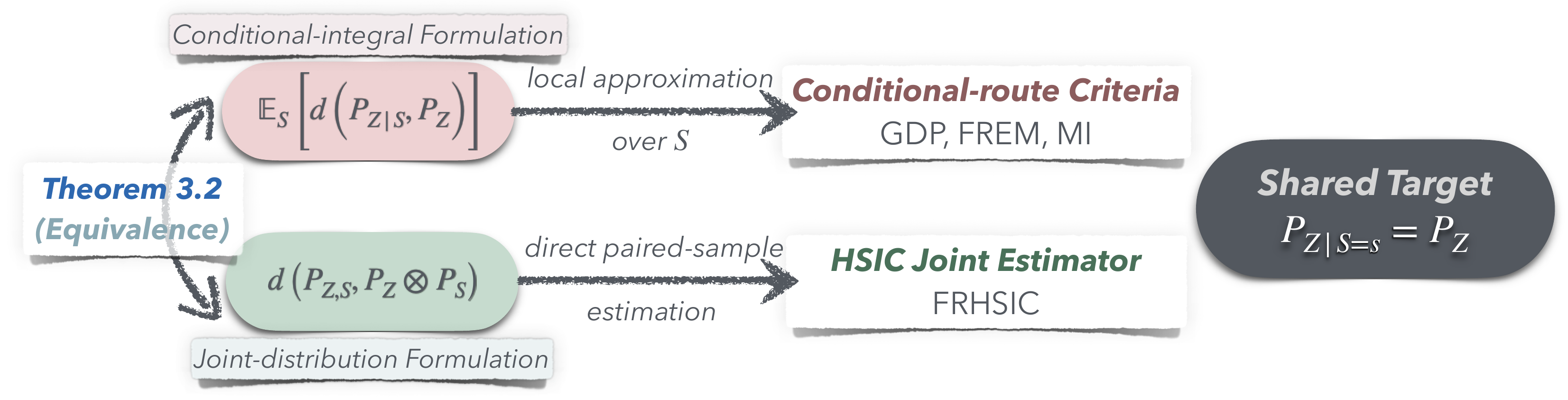}
    \caption{Conditional and joint routes for continuous-sensitive fair representation learning. The conditional route approximates local objects indexed by $S=s$; the joint route estimates dependence directly from paired samples.}
    \label{fig:routes}
\end{figure}

%======================================================================
\section{Preliminaries and Existing Criteria}
\label{sec:preliminary}
%======================================================================

This section fixes notation, recalls the integral probability metrics used throughout, and shows that the existing continuous-$S$ fairness criteria share a common conditional-integral form, the baseline against which our joint route is compared.

\subsection{Setup and representation-level fairness}

Representation-level fairness asks for a representation that is statistically independent of the sensitive attribute while remaining predictive of the target.
Let $X \in \mathcal{X}$ denote the input random vector, $S \in \mathcal{S} \subseteq \mathbb{R}$ the continuous sensitive attribute, and $Y \in \mathcal{Y}$ the target variable.
A representation function $h:\mathcal X\to\mathcal Z$ yields $Z=h(X)$, and a prediction head $f:\mathcal Z\to\mathcal Y$ yields $\widehat Y=f(Z)=f\circ h(X)$.
Write $P_Z$ for the marginal law of $Z$, $P_{Z\mid S=s}$ for the conditional law of $Z$ given $S=s$, and $P_{Z,S}$ for the joint law.
For a reproducing kernel Hilbert space (RKHS) $\mathcal{F}$ with kernel $k$, the mean embedding of a distribution $P$ is $\mu_P=\mathbb E_{X\sim P}[k(X,\cdot)]\in\mathcal F$.
For binary $S\in\{0,1\}$, a predictor $\widehat Y$ satisfies demographic parity (DP) if $\widehat Y\perp S$, which for binary $\widehat Y$ reduces to $\mathbb E[\widehat Y\mid S=0]=\mathbb E[\widehat Y\mid S=1]$.
The goal of fair representation learning is a representation $Z=h(X)$ that is independent of $S$ while remaining informative about $Y$. For continuous $S$ the exact representation-level target is
\begin{equation}
\label{eq:frl_continuous}
    Z\perp S,
    \qquad\text{equivalently}\qquad
    P_{Z\mid S=s}=P_Z
    \quad\text{for }P_S\text{-almost every }s
\end{equation}
for any regular conditional distribution $\{P_{Z\mid S=s}\}_{s\in\cs}$.
This independence condition is the target that every criterion in this paper enforces; empirical criteria approximate it through a discrepancy between $P_{Z\mid S=s}$ and $P_Z$, introduced next.

\subsection{Discrepancies between probability laws}

Distances between distributions in this paper are integral probability metrics, and the maximum mean discrepancy is the kernel instance that, for a characteristic kernel, vanishes exactly when two laws coincide.
For a class $\mathcal F$ of bounded measurable functions, the integral probability metric (IPM) between distributions $P$ and $Q$ is $\mathrm{IPM}_{\mathcal F}(P,Q):=\sup_{f\in\mathcal F}|\mathbb E_P[f]-\mathbb E_Q[f]|$.
Taking $\mathcal F$ to be the unit ball of an RKHS gives the maximum mean discrepancy,
\[
    \mathrm{MMD}(P,Q)=\sup_{\|f\|_{\mathcal F}\le 1}|\mathbb E_P[f]-\mathbb E_Q[f]|=\|\mu_P-\mu_Q\|_{\mathcal F}.
\]
When $k$ is characteristic, $\mathrm{MMD}(P,Q)=0$ if and only if $P=Q$ \citep{gretton2012kernel}.
Section~\ref{sec:joint-view} applies the IPM with $P=P_{Z,S}$ and $Q=P_Z\otimes P_S$.

\subsection{Conditional-integral criteria for continuous sensitive attributes}

Existing continuous-$S$ fairness criteria share one form, the $S$-average of a per-value discrepancy between the conditional and marginal laws of $Z$, and differ only in that discrepancy.
For a discrepancy $d$ between distributions, the integral functional is
\[
    \mathcal I_d(Z;S):=\mathbb E_S\!\left[d(P_{Z\mid S},P_Z)\right],
\]
which is well-defined for Polish $\mathcal Z,\mathcal S$ once a regular conditional probability is selected; existence follows from the disintegration theorem \citep[Theorem~5.4]{kallenberg2002foundations}.
Three instances recur.
Generalized demographic parity (GDP) \citep{jiang2022generalized} takes $d$ to be a first-moment difference evaluated through a measurable head $f:\mathcal Z\to\mathbb R$,
\begin{equation}
\label{eq:GDP}
    \Delta_{\mathrm{GDP}}(f)=\mathbb E_S\left|\mathbb E_Z[f(Z)\mid S]-\mathbb E_Z[f(Z)]\right|,\qquad Z=h(X),
\end{equation}
with the conditional expectation estimated by Nadaraya--Watson kernel smoothing on $S$ \citep{nadaraya1964estimating,watson1964smooth}; GDP reduces to DP when $S$ is binary.
The expectation of IPM (EIPM) \citep{kong2025fair} takes $d$ to be an IPM,
\[
    \mathrm{EIPM}_{\mathcal V}(Z;S):=\mathbb E_S[\mathrm{IPM}_{\mathcal V}(P_{Z\mid S},P_Z)],\qquad \mathrm{IPM}_{\mathcal V}(P,Q)=\sup_{f\in\mathcal V}|\mathbb E_P[f]-\mathbb E_Q[f]|,
\]
and controlling EIPM controls GDP \citep{kong2025fair}.
Mutual information is the Kullback--Leibler instance,
\[
    I(Z;S)=\mathbb E_S\!\left[\mathrm{KL}(P_{Z\mid S}\,\|\,P_Z)\right],
\]
and characterizes independence, $I(Z;S)=0\iff Z\perp S$ \citep[Theorem~2.6.3]{coverthomas2006}.
Section~\ref{sec:joint-view} relates this conditional-integral route to a single joint-vs-product discrepancy through disintegration, recovering the conditional-integral IPM on decomposable witness classes.

\subsection{Joint dependence measures and HSIC}

An alternative measures dependence directly between the joint law and the product of marginals, and the Hilbert--Schmidt independence criterion is the kernel instance that, under characteristic product kernels, vanishes exactly under independence.
This is the object the method of this paper is built on.

\begin{definition}[Mean embeddings and HSIC]\label{def:hsic}
Let $\cf_\cz$ and $\cf_\cs$ be RKHSs on $\cz$ and $\cs$ with bounded characteristic kernels $k_\cz$ and $k_\cs$, and let $\cf_{\cz\otimes\cs}$ be the tensor-product RKHS on $\cz\times\cs$ with product kernel $k_\cz\otimes k_\cs$.
The marginal and joint mean embeddings are
\[
    \mu_Z:=\mathbb E_Z[k_\cz(Z,\cdot)]\in\cf_\cz,\qquad
    \mu_S:=\mathbb E_S[k_\cs(S,\cdot)]\in\cf_\cs,
\]
\[
    \mu_{Z,S}:=\mathbb E_{(Z,S)}\bigl[k_\cz(Z,\cdot)\otimes k_\cs(S,\cdot)\bigr]\in\cf_{\cz\otimes\cs},
\]
and $\mu_Z\otimes\mu_S\in\cf_{\cz\otimes\cs}$ is the embedding of the product law $P_Z\otimes P_S$.
The conditional mean embedding of $Z$ given $S=s$ is $\mu_{Z\mid S=s}:=\mathbb E[k_\cz(Z,\cdot)\mid S=s]\in\cf_\cz$; we write $\mu_{Z\mid S}$ for the $\cf_\cz$-valued random element $s\mapsto\mu_{Z\mid S=s}$ evaluated at $S$.
The Hilbert--Schmidt independence criterion (HSIC) is the squared maximum mean discrepancy between the joint law and the product of marginals,
\[
    \hsic(Z,S):=\bigl\|\mu_{Z,S}-\mu_Z\otimes\mu_S\bigr\|^2_{\cf_{\cz\otimes\cs}}.
\]
\end{definition}

Under characteristic product kernels, $\hsic(Z,S)=0$ if and only if $Z\perp S$ \citep{gretton2005measuring}.
Section~\ref{sec:joint-view} relates this joint discrepancy to conditional-integral criteria through a disintegration identity. The equality with the conditional-integral IPM requires the decomposability condition stated in Corollary~\ref{cor:decomposable-class}.

%======================================================================
\section{The Joint Discrepancy: An Intrinsic Metric for the Conditional-Integral Criteria}\label{sec:joint-view}
%======================================================================

\subsection{Zero-level relation to existing criteria}
\label{sec:zero-set-relation}

This section develops the central object of this paper: a single joint discrepancy between $P_{Z,S}$ and $P_Z\otimes P_S$ that serves as an intrinsic metric for the conditional-integral criteria. We begin at the zero level, where a characteristic witness class makes the joint discrepancy vanish exactly at the representation-level fairness target shared by the existing criteria; the nonzero regime is developed in \S\ref{sec:framework}--\S\ref{sec:hsic-gap-control}.

\begin{proposition}[Zero-level characterization of representation fairness]
\label{thm:subgroup_equivalence}\label{prop:equivalence}\label{prop:hsic_gdp_population}
Let $\mathcal Z$ and $\mathcal S$ be Polish spaces, let $\mathcal F$ be a characteristic class of bounded measurable functions on $\mathcal Z\times\mathcal S$, that is, one for which $\mathrm{IPM}_{\mathcal F}$ separates probability laws on $\mathcal Z\times\mathcal S$, and let $(Z,S)$ be a random pair on $\mathcal Z\times\mathcal S$. Then the following are equivalent:
\begin{enumerate}
    \item $\mathrm{IPM}_{\mathcal F}(P_{Z,S},\,P_Z\otimes P_S)=0$;
    \item $Z\perp S$, equivalently $P_{Z,S}=P_Z\otimes P_S$;
    \item for some, equivalently every, regular conditional distribution $\{P_{Z\mid S=s}\}_{s\in\mathcal S}$,
    \[
        P_{Z\mid S=s}=P_Z
        \quad\text{for }P_S\text{-a.e. }s;
    \]
    \item for every bounded measurable $g:\mathcal Z\to\mathbb R$,
    \[
        \mathbb E[g(Z)\mid S=s]=\mathbb E[g(Z)]
        \quad\text{for }P_S\text{-a.e. }s.
    \]
\end{enumerate}
\end{proposition}

\begin{proof}
{\sloppy
The equivalences are proved in Appendix~\ref{app:pf_thm_subgroup_equivalence}: condition (1) is equivalent to (2) because a characteristic class separates $P_{Z,S}$ from $P_Z\otimes P_S$, and (2)--(4) are equivalent by disintegration. The MMD-based EIPM instance, $\mathrm{EIPM}_{\mathcal V}(Z;S)=0$ for the unit ball $\mathcal V\subset\cf_\cz$, and the downstream-head (GDP) form are recorded in Appendices~\ref{app:pf_prop_equivalence}--\ref{app:pf_cor_hsic_gdp_population}.\par}
\end{proof}

\par\smallskip\noindent
\textit{Remark (zero-level equivalence only).}\enspace Proposition~\ref{thm:subgroup_equivalence} is a zero-level statement. Away from zero the joint discrepancy and the conditional-integral criteria are different functionals: the conditional-integral criteria average a per-value discrepancy between $P_{Z\mid S=s}$ and $P_Z$ over the law of $S$, whereas the joint discrepancy applies a single supremum to the paired law.
\par\smallskip

\subsection{Conditional-integral and joint formulations}\label{sec:joint-integral-equivalence}\label{sec:framework}

Conditional-integral criteria compare $P_{Z\mid S=s}$ with $P_Z$ and average the comparison over $S$, whereas the joint route compares the joint law $P_{Z,S}$ with the product of marginals $P_Z\otimes P_S$.
When $S$ takes finitely many values the two formulations coincide as aggregated two-sample MMDs, recorded as Lemma~\ref{lem:discrete-coincidence} in Appendix~\ref{app:discrete-coincidence}.
The case of practical interest is continuous $S$, where the conditional family $\{P_{Z\mid S=s}\}_s$ is uncountable and the two formulations differ in how they are estimated.
The next theorem disintegrates the joint-vs-product IPM into a $P_S$-weighted conditional contrast.

\begin{theorem}[Disintegration of the joint-vs-product IPM]\label{thm:joint-equals-integral}
Let $\cz$ and $\cs$ be standard Borel spaces, and let $\{P_{Z\mid S=s}\}_{s\in\cs}$ be a regular conditional distribution of $Z$ given $S$. For any class $\mathcal F$ of bounded measurable functions on $\cz\times\cs$,
\[
\mathrm{IPM}_{\mathcal F}\bigl(P_{Z,S},\,P_Z\otimes P_S\bigr)
=\sup_{f\in\mathcal F}\left|\int_{\cs}\!\int_{\cz} f(z,s)\,d\bigl(P_{Z\mid S=s}-P_Z\bigr)(z)\,dP_S(ds)\right|.
\]
\end{theorem}

\begin{proof}
See Appendix~\ref{app:pf_thm_joint_disintegration}.
\end{proof}

The theorem shows that the $P_S$ weighting arises from disintegrating the joint law, rather than from an external weighting choice.
For any class $\mathcal F$, it places a single supremum outside the $P_S$-average of the conditional contrast; this average equals the conditional-integral functional $\mathcal I_d$ exactly when $\mathcal F$ decomposes across $s$, which Definition~\ref{def:decomposable-class} and Corollary~\ref{cor:decomposable-class} make precise.

\begin{definition}[Decomposable witness class]
\label{def:decomposable-class}
Let $\mathcal V$ be a symmetric class of bounded measurable functions on $\mathcal Z$, meaning $g\in\mathcal V$ implies $-g\in\mathcal V$. A class $\mathcal F$ of bounded measurable functions on $\mathcal Z\times\mathcal S$ is decomposable over $\mathcal V$ if:
\begin{enumerate}
    \item for every measurable selector $s\mapsto g_s\in\mathcal V$, the function $(z,s)\mapsto g_s(z)$ belongs to $\mathcal F$;
    \item for every $f\in\mathcal F$, the slice $f(\cdot,s)$ belongs to $\mathcal V$ for $P_S$-almost every $s$.
\end{enumerate}
\end{definition}

\begin{corollary}[Recovery of the conditional-integral IPM under decomposability]
\label{cor:decomposable-class}
Adopt the assumptions of Theorem~\ref{thm:joint-equals-integral}: $\cz$ and $\cs$ are standard Borel spaces and $\{P_{Z\mid S=s}\}_{s\in\cs}$ is a regular conditional distribution of $Z$ given $S$. Let $\mathcal V$ be a symmetric uniformly bounded class of measurable functions on $\mathcal Z$, and define
\[
d_{\mathcal V}(P,Q):=\mathrm{IPM}_{\mathcal V}(P,Q).
\]
Let $\mathcal F$ be decomposable over $\mathcal V$ in the sense of Definition~\ref{def:decomposable-class}. Assume that for every $\varepsilon>0$ there exists a measurable selector $s\mapsto g_s^\varepsilon\in\mathcal V$ such that
\[
\int g_s^\varepsilon(z)\,d(P_{Z\mid S=s}-P_Z)(z)
\ge
d_{\mathcal V}(P_{Z\mid S=s},P_Z)-\varepsilon
\quad
\text{for }P_S\text{-a.e. }s.
\]
Then
\[
\mathrm{IPM}_{\mathcal F}(P_{Z,S},P_Z\otimes P_S)
=
\mathbb E_S\!\left[d_{\mathcal V}(P_{Z\mid S=s},P_Z)\right]
=
\mathcal I_{d_{\mathcal V}}(Z;S).
\]
\end{corollary}

\begin{proof}
See Appendix~\ref{app:pf_cor_decomposable}.
\end{proof}

Corollary~\ref{cor:decomposable-class} therefore recovers the conditional-integral IPM, which GDP and EIPM instantiate, exactly when the witness class decomposes across $s$; mutual information is the Kullback--Leibler instance of the same $S$-averaged form.

%======================================================================
\section{HSIC: A Closed-Form Equivalent with Demographic-Parity Control}\label{sec:method}
%======================================================================

HSIC is equivalent, up to an explicit spectral tail, to the conditional MMD integral it replaces, and that equivalence carries over to the fairness metric itself. \S\ref{sec:hsic-instance} proves the equivalence with the conditional MMD integral, the conditional-integral criterion HSIC replaces. \S\ref{sec:hsic-gap-control} turns it into control of the demographic-parity gap at the population level and on a finite sample.

\subsection{Equivalence with the conditional MMD integral}\label{sec:hsic-instance}

HSIC is the kernel joint discrepancy obtained by comparing $P_{Z,S}$ with $P_Z\otimes P_S$.
Instantiating the joint IPM of Theorem~\ref{thm:joint-equals-integral} with the unit ball of the tensor-product RKHS $\cf_{\cz\otimes\cs}$ of Definition~\ref{def:hsic} gives the maximum mean discrepancy $\mathrm{MMD}_{k_\cz\otimes k_\cs}(P_{Z,S},P_Z\otimes P_S)$, and HSIC is its square,
\[
\hsic(Z,S)=\mathrm{MMD}^2_{k_\cz\otimes k_\cs}\bigl(P_{Z,S},\,P_Z\otimes P_S\bigr).
\]
We assume throughout that the product kernel $k_\cz \otimes k_\cs$ is characteristic to the relevant joint and product measures on $\cz \times \cs$; this holds under standard conditions for the Gaussian kernels used in our experiments.
The first of the two bounds comes from a decomposition. The following proposition writes HSIC over the conditional deviations $\mu_{Z\mid S=s}-\mu_Z$ and bounds it by the squared expected MMD, so a small conditional MMD makes HSIC small.

\begin{proposition}[Decomposition of HSIC and its expected-MMD bound]
\label{thm:EIPM_HSIC}
    Suppose $k_\cz$ and $k_\cs$ are bounded, i.e., $k_\cz(z,z) \leq \kappa_\cz$ and $k_\cs(s,s) \leq \kappa_\cs$ for all $z \in \cz$, $s \in \cs$.
    Then HSIC admits the decomposition
    \begin{equation}
    \label{eq:hsic_mmd}
        \mathrm{HSIC}(Z, S) = \mathbb{E}_{S, S'} \left[\left\langle\mu_{Z | S}-\mu_Z,\mu_{Z | S^{\prime}}-\mu_Z\right\rangle_{\mathcal{F}_{\mathcal{Z}}} \cdot k_{\mathcal{S}}(S, S') \right],
    \end{equation}
    where $S, S'$ are independent copies.
    Moreover, with $\mathrm{MMD}_{k_\cz}(P_{Z\mid S=s},P_Z)=\|\mu_{Z\mid S=s}-\mu_Z\|_{\cf_\cz}$ the maximum mean discrepancy of \S\ref{sec:preliminary} under $k_\cz$, HSIC is bounded by the squared expected MMD:
    \begin{equation}
    \label{eq:hsic_conditional}
        \mathrm{HSIC}(Z, S) \leq \kappa_\mathcal{S}\cdot\left(\mathbb{E}_{S}\!\left[
        \mathrm{MMD}_{k_\cz}\bigl(P_{Z\mid S},P_Z\bigr) \right]\right)^2.
    \end{equation}
\end{proposition}

\begin{proof}
See Appendix~\ref{app:pf_thm_eipm_hsic}.
\end{proof}

\begin{remark}[Relation to the MMD-based EIPM]
\label{rem:hsic-eipm-direction}
The expected MMD $\mathbb E_S[\mathrm{MMD}_{k_\cz}(P_{Z\mid S},P_Z)]$ is the MMD-based EIPM, the expectation-of-IPM criterion realized with the RKHS witness class $\mathcal V\subset\cf_\cz$. Equation~\eqref{eq:hsic_mmd} writes HSIC as a covariance-weighted aggregation of the conditional mean-embedding deviations, and Inequality~\eqref{eq:hsic_conditional} bounds HSIC by $\kappa_\cs$ times its square. The reverse control, of the conditional MMD integral $\mathbb E_S[\mathrm{MMD}_{k_\cz}^2(P_{Z\mid S},P_Z)]$ by HSIC up to an explicit spectral tail, is Theorem~\ref{cor:hsic-mmd-equivalence}.
\end{remark}

The second bound comes from the same decomposition. Because $\|\mu_{Z\mid S=s}-\mu_Z\|_{\cf_\cz}=\mathrm{MMD}_{k_\cz}(P_{Z\mid S=s},P_Z)$, the conditional deviation $\Delta(s):=\mu_{Z\mid S=s}-\mu_Z$ is the conditional MMD in vector form. The next theorem shows that a small HSIC makes every RKHS-smoothed average of $\Delta$ small, so the joint statistic controls the conditional MMD after smoothing.

\begin{theorem}[Smoothed conditional-distribution control by HSIC]
\label{thm:hsic_smoothed_conditional_control}
Let $k_\cz$ and $k_\cs$ be bounded measurable kernels with RKHSs $\cf_\cz$ and $\cf_\cs$. Let
\[
    \Delta(s):=\mu_{Z\mid S=s}-\mu_Z\in\cf_\cz,
    \qquad
    \mu_{Z\mid S=s}:=\mathbb E[k_\cz(Z,\cdot)\mid S=s].
\]
Then $\Delta\in L^2(P_S;\cf_\cz)$ and
\[
    \hsic(Z,S)
    =
    \bigl\|\mathbb E_S[\Delta(S)\otimes k_\cs(S,\cdot)]\bigr\|_{\cf_{\cz\otimes\cs}}^2 .
\]
Moreover, for every $g\in\cf_\cs$,
\begin{equation}
\label{eq:hsic_controls_weighted_embedding}
    \bigl\|\mathbb E_S[g(S)\Delta(S)]\bigr\|_{\cf_\cz}^2
    \le
    \|g\|_{\cf_\cs}^2\,\hsic(Z,S).
\end{equation}
\end{theorem}

\begin{proof}
See Appendix~\ref{app:pf_prop_hsic_smooth_control}.
\end{proof}

At a single sensitive value $s_0$, the same bound controls the conditional MMD localized near $s_0$, the contrast a per-value estimator would target.

\begin{corollary}[Localized conditional-MMD control by HSIC]
\label{cor:localized_sensitive_contrasts}
Under the assumptions of Theorem~\ref{thm:hsic_smoothed_conditional_control},
for any $s_0\in\cs$ with $k_\cs(s_0,s_0)>0$,
\[
    \bigl\|
        \mathbb E_S[
            k_\cs(S,s_0)(\mu_{Z\mid S}-\mu_Z)
        ]
    \bigr\|_{\cf_\cz}^2
    \le
    k_\cs(s_0,s_0)\,\hsic(Z,S).
\]
More generally, for any $s_0,s_1\in\cs$,
\[
\begin{aligned}
    \bigl\|
        \mathbb E_S[
            \{k_\cs(S,s_0)-k_\cs(S,s_1)\}
            (\mu_{Z\mid S}-\mu_Z)
        ]
    \bigr\|_{\cf_\cz}^2
    \le
    \|k_\cs(s_0,\cdot)-k_\cs(s_1,\cdot)\|_{\cf_\cs}^2
    \hsic(Z,S).
\end{aligned}
\]
\end{corollary}

\begin{proof}
See Appendix~\ref{app:pf_cor_localized}.
\end{proof}

Resolving $\Delta$ along the spectrum of the sensitive-attribute kernel turns the control of Theorem~\ref{thm:hsic_smoothed_conditional_control} into a bound on the conditional MMD integral; combined with the bound of Proposition~\ref{thm:EIPM_HSIC}, the two give a two-sided equivalence: HSIC and the conditional MMD integral control each other up to an explicit spectral tail, so making the closed-form statistic small makes the conditional MMD integral that EIPM targets small up to the residual tail $\rho_m^2$.

\begin{theorem}[Spectral equivalence of HSIC and the conditional MMD integral]
\label{cor:hsic-mmd-equivalence}
Assume the conditions of Theorem~\ref{thm:hsic_smoothed_conditional_control}, and let $\kappa_\cs$ be the bound on $k_\cs$ from Proposition~\ref{thm:EIPM_HSIC}. Let $\Delta(s):=\mu_{Z\mid S=s}-\mu_Z$ and let $T_S:L^2(P_S)\to L^2(P_S)$ be the kernel integral operator
\[
    (T_S u)(s):=\int k_\cs(s,t)u(t)\,dP_S(t),
\]
with orthonormal eigensystem $(\lambda_\ell,\psi_\ell)_{\ell\ge1}$, $\lambda_1\ge\lambda_2\ge\cdots\ge0$, and let $P_m$ denote the $L^2(P_S)$ projection onto $\mathrm{span}\{\psi_1,\ldots,\psi_m\}$, extended componentwise to $L^2(P_S;\cf_\cz)$. Then, for any $m$ with $\lambda_m>0$ such that $\|(I-P_m)\Delta\|_{L^2(P_S;\cf_\cz)}^2\le \rho_m^2$,
\[
    \frac{1}{\kappa_\cs}\,\hsic(Z,S)
    \;\le\;
    \mathbb E_S\!\left[\mathrm{MMD}_{k_\cz}^2(P_{Z\mid S},P_Z)\right]
    \;\le\;
    \frac{1}{\lambda_m}\,\hsic(Z,S)+\rho_m^2 .
\]
Hence $\hsic(Z,S)$ and the conditional MMD integral control each other up to the constants $\kappa_\cs,\lambda_m$ and the explicit spectral tail $\rho_m^2$.
\end{theorem}

\begin{proof}
The upper bound is the spectral argument of Appendix~\ref{app:pf_cor_spectral_conditional_mmd}: with $\mathbb E_S[\mathrm{MMD}_{k_\cz}^2(P_{Z\mid S},P_Z)]=\|\Delta\|_{L^2(P_S;\cf_\cz)}^2$, expanding $\Delta$ in the eigenbasis of $T_S$ gives $\|\Delta\|_{L^2(P_S;\cf_\cz)}^2\le\lambda_m^{-1}\hsic(Z,S)+\rho_m^2$. For the lower bound, Proposition~\ref{thm:EIPM_HSIC} gives $\hsic(Z,S)\le\kappa_\cs\bigl(\mathbb E_S[\mathrm{MMD}_{k_\cz}(P_{Z\mid S},P_Z)]\bigr)^2$, and Jensen's inequality gives $\bigl(\mathbb E_S[\mathrm{MMD}_{k_\cz}(P_{Z\mid S},P_Z)]\bigr)^2\le\mathbb E_S[\mathrm{MMD}_{k_\cz}^2(P_{Z\mid S},P_Z)]$.
\end{proof}

HSIC controls the conditional-distribution deviations lying in the well-conditioned spectral directions of $T_S$. Because the kernel integral operator on $\cs$ is compact, the projected tail $\rho_m$ decreases in $m$ and the remaining directions enter only through it, so when $\rho_m$ is small, small HSIC implies small conditional-integral MMD. This is the population counterpart of the finite-sample spectral factor $\hat\lambda_m^{-1}$ in Theorem~\ref{thm:hsic_gdp}.

\subsection{Control of the demographic-parity gap}\label{sec:hsic-gap-control}

The equivalence transfers to the fairness metric, in population and then on a finite sample. At the population level, the same spectral bound reaches generalized demographic parity: for an RKHS prediction head, small HSIC controls $\Delta_{\mathrm{GDP}}(f)$ up to the spectral tail of its conditional mean $m_f$.

\begin{corollary}[Spectral control of GDP for RKHS heads]
\label{cor:spectral_gdp_control}
Under the assumptions of Theorem~\ref{cor:hsic-mmd-equivalence}, let $f\in\cf_\cz$ and define
\[
    m_f(s):=\mathbb E[f(Z)\mid S=s]-\mathbb E[f(Z)].
\]
Let $P_m$ be the spectral projection from Theorem~\ref{cor:hsic-mmd-equivalence}. Then
\[
    \|P_m m_f\|_{L^2(P_S)}^2
    \le
    \frac{\|f\|_{\cf_\cz}^2}{\lambda_m}\,\hsic(Z,S).
\]
If $\|(I-P_m)m_f\|_{L^2(P_S)}^2\le r_{m,f}^2$, then
\[
    \Delta_{\mathrm{GDP}}(f)^2
    =
    \bigl(\mathbb E_S|m_f(S)|\bigr)^2
    \le
    \frac{\|f\|_{\cf_\cz}^2}{\lambda_m}\,\hsic(Z,S)
    +
    r_{m,f}^2.
\]
\end{corollary}

\begin{proof}
See Appendix~\ref{app:pf_cor_spectral_gdp}.
\end{proof}

On a finite sample, the same control holds for the gap computed from data. The next theorem is the finite-sample form of Corollary~\ref{cor:spectral_gdp_control}: the empirical HSIC statistic controls the empirical demographic-parity gap of any RKHS head, with the eigenvalue $\hat\lambda_m$ of the centered sensitive Gram matrix $\widetilde L$ in the role of $\lambda_m$.

\begin{theorem}[Finite-sample control of the empirical demographic-parity gap by HSIC]
\label{thm:hsic_gdp}
Let $\{(Z_i,S_i)\}_{i=1}^n$ be a sample. Let $K,L\in\mathbb R^{n\times n}$ be the Gram matrices
\[
    K_{ij}=k_\cz(Z_i,Z_j),
    \qquad
    L_{ij}=k_\cs(S_i,S_j),
\]
and let
\[
    H:=I_n-n^{-1}\mathbf 1\mathbf 1^\top,
    \qquad
    \widetilde K:=HKH,
    \qquad
    \widetilde L:=HLH.
\]
Define the biased empirical HSIC statistic
\begin{equation}
\label{eq:hsic_empirical}
    \widehat{\hsic}_n(Z,S):=n^{-2}\operatorname{tr}(\widetilde K\widetilde L).
\end{equation}
Let $\hat\lambda_1\ge\cdots\ge\hat\lambda_r>0$ be the positive eigenvalues of $n^{-1}\widetilde L$ with orthonormal eigenvectors $\hat e_1,\ldots,\hat e_r\in\mathbb R^n$, where $r=\operatorname{rank}(\widetilde L)$, and for $1\le m\le r$ let $P_m$ be the orthogonal projection onto $\operatorname{span}\{\hat e_1,\ldots,\hat e_m\}$. For $f\in\cf_\cz$, define
\[
    \bar f_n:=n^{-1}\sum_{j=1}^n f(Z_j),
    \qquad
    \hat\delta_{f,i}:=f(Z_i)-\bar f_n,
    \qquad
    \bm\delta_f:=(\hat\delta_{f,1},\ldots,\hat\delta_{f,n})^\top,
\]
and let $\widehat m_f:=P_m\bm\delta_f$ be the empirical conditional-mean gap of $f$, the prediction's dependence on $S$ resolved on the $m$ best-conditioned sensitive directions. Then the empirical demographic-parity gap $\widehat\Delta_{\mathrm{GDP}}(f):=n^{-1}\sum_{i=1}^n\lvert(\widehat m_f)_i\rvert$ satisfies
\begin{equation}
\label{eq:hsic_gdp_bound}
    \widehat\Delta_{\mathrm{GDP}}(f)^2
    \;\le\;
    \frac1n\sum_{i=1}^n(\widehat m_f)_i^2
    \;\le\;
    \frac{\|f\|_{\cf_\cz}^2}{\hat\lambda_m}\,\widehat{\hsic}_n(Z,S).
\end{equation}
\end{theorem}

\begin{proof}
See Appendix~\ref{app:pf_thm_hsic_gdp}.
\end{proof}

\paragraph{Interpretation}
Equation~\eqref{eq:hsic_gdp_bound} is the term-by-term empirical form of Corollary~\ref{cor:spectral_gdp_control}: the centered Gram matrix $\widetilde L$ replaces the operator $T_S$, its $m$-th eigenvalue $\hat\lambda_m$ replaces $\lambda_m$, and the projection $\widehat m_f=P_m\bm\delta_f$ of the prediction onto the $m$ best-conditioned sensitive directions replaces the projected conditional mean $P_m m_f$. The free parameter $m$ trades resolution for conditioning, as in the population bound: a larger $m$ resolves more of the prediction's dependence on $S$ but lowers $\hat\lambda_m$. A small $\widehat{\hsic}_n$ therefore bounds $\widehat\Delta_{\mathrm{GDP}}(f)$ for every RKHS head $f$, through the head-dependent factor $\|f\|_{\cf_\cz}^2/\hat\lambda_m$.

%======================================================================
\section{FRHSIC: A Regularizer with a Faster Train-to-Population Rate}\label{sec:frhsic}
%======================================================================

FRHSIC turns the population control of the previous section into a trainable penalty whose empirical value bounds the learned encoder's population dependence at the $n^{-1/2}$ rate. It is the empirical objective~\eqref{eq:frhsic_objective} together with its minibatch training algorithm: \S\ref{sec:frhsic-objective} defines the objective, and \S\ref{sec:uniform-train-pop} gives the training algorithm (the \emph{Implementation} paragraph) and the uniform train-to-population guarantee. The joint-distribution framework underlying the HSIC penalty was developed in Section~\ref{sec:joint-view}; this section is self-contained for a reader who wants only the method and its statistical guarantee.

\subsection{Empirical FRHSIC objective}\label{sec:frhsic-objective}

FRHSIC trains the representation $Z:=h(X)$ (as in \S\ref{sec:preliminary}) by penalizing empirical dependence between $Z$ and $S$ over a feasible encoder class $\mathcal H$ and prediction-head class $\mathcal F$. With fixed bounded characteristic kernels $k_\cz,k_\cs$, the empirical HSIC penalty $\widehat{\hsic}_n(h(X),S)$ is the biased $V$-statistic~\eqref{eq:hsic_empirical} formed from the paired sample $\{(h(X_i),S_i)\}_{i=1}^n$. FRHSIC solves the empirical regularized problem
\begin{equation}
\label{eq:frhsic_objective}
    \min_{h\in\mathcal H,\ f\in\mathcal F}
    \widehat{\mathcal L}_n(f\circ h)
    +
    \lambda\,\widehat{\hsic}_n(h(X),S),
\end{equation}
where $\widehat{\mathcal L}_n$ is the empirical prediction loss (cross-entropy for classification, MSE for regression) and $\lambda>0$ controls the fairness--accuracy tradeoff. Equation~\eqref{eq:frhsic_objective} is the Lagrangian relaxation of the empirical constrained problem $\min_{h,f}\widehat{\mathcal L}_n(f\circ h)$ subject to $\widehat{\hsic}_n(h(X),S)\le\delta$. We use the biased estimator because it is nonnegative and hence a valid penalty; the unbiased $U$-statistic estimator \citep{song2012feature} can be negative in finite samples, though both are consistent.

The penalty is a joint-sample statistic on $\{(h(X_i),S_i)\}$ and does not construct local estimates of $P_{Z\mid S=s}$. The concentration result below is stated for fixed kernels. In the experiments the kernel bandwidths are set by the median heuristic and the resulting statistic is used as a practical training penalty; the theorem isolates the fixed-kernel statistical behavior.

\subsection{Uniform train-to-population control}\label{sec:uniform-train-pop}

Empirical HSIC concentrates on its population value uniformly over the encoder class $\mathcal H$, so the data-dependent encoder that minimizes the FRHSIC penalty still has small population dependence, up to a uniform error of order $n^{-1/2}$ under controlled encoder complexity. Uniformity is essential because the penalty is evaluated on an encoder selected from the same sample, so a pointwise concentration bound for a fixed encoder would not suffice.

\begin{theorem}[Uniform concentration of empirical HSIC over the encoder class]
\label{thm:uniform-hsic-frhsic}
Let $\{(X_i,S_i)\}_{i=1}^n$ be i.i.d.\ copies of $(X,S)$, and let $\mathcal H$ be a class of encoders $h:\mathcal X\to\cz\subseteq\mathbb R^{d_\cz}$. For $h\in\mathcal H$ write $Z^h:=h(X)$ and $Z_i^h:=h(X_i)$. Let $k_\cz,k_\cs$ be fixed kernels with feature maps $\phi_\cz(z):=k_\cz(z,\cdot)\in\cf_\cz$ and $\phi_\cs(s):=k_\cs(s,\cdot)\in\cf_\cs$, satisfying the boundedness and Lipschitz conditions
\[
    \sup_{z\in\cz}k_\cz(z,z)\le\nu_\cz,
    \qquad
    \sup_{s\in\cs}k_\cs(s,s)\le\nu_\cs,
\]
\[
    \|\phi_\cz(z)-\phi_\cz(z')\|_{\cf_\cz}\le\ell_\cz\,\|z-z'\|,
    \qquad
    \|\phi_\cs(s)-\phi_\cs(s')\|_{\cf_\cs}\le\ell_\cs\,\|s-s'\|,
\]
where $k_\cz(z,z)=\|\phi_\cz(z)\|_{\cf_\cz}^2$ and $k_\cs(s,s)=\|\phi_\cs(s)\|_{\cf_\cs}^2$. For $h\in\mathcal H$, write $\widehat{\hsic}_n(h)$ for the biased empirical HSIC estimator~\eqref{eq:hsic_empirical} computed from the paired sample $\{(Z_i^h,S_i)\}_{i=1}^n$, and $\hsic(h)$ for the population HSIC of Definition~\ref{def:hsic} with $Z=Z^h$; these are the estimator and functional already defined, evaluated at the encoder-transformed representation $Z^h$ with the fixed kernels $k_\cz,k_\cs$ unchanged. Let
\[
    \widehat{\mathcal G}_n(\mathcal H):=\mathbb E_{\xi}\!\Bigl[\sup_{h\in\mathcal H}\tfrac1n\textstyle\sum_{i=1}^n\langle\xi_i,h(X_i)\rangle\,\Big|\,X_1,\ldots,X_n\Bigr],
    \quad
    \xi_i\stackrel{\mathrm{iid}}{\sim}\mathcal N(0,I_{d_\cz}),
\]
be the empirical Gaussian complexity of $\mathcal H$, and set $\mathcal G_n(\mathcal H):=\mathbb E[\widehat{\mathcal G}_n(\mathcal H)]$. Define
\begin{equation}
\label{eq:Bn_explicit}
    B_n(\mathcal H,\delta):=8\sqrt2\,\nu_\cz\nu_\cs\sqrt{\tfrac{\log(2/\delta)}{n}}+\frac{4\,\nu_\cz\nu_\cs}{n}+48\sqrt\pi\,\max\{\nu_\cz\ell_\cs,\nu_\cs\ell_\cz\}\,\mathcal G_n(\mathcal H).
\end{equation}
Then, for every $\delta\in(0,1)$, with probability at least $1-\delta$,
\begin{equation}
\label{eq:uniform-hsic-bound}
    \sup_{h\in\mathcal H}\bigl|\widehat{\hsic}_n(h)-\hsic(h)\bigr|\le B_n(\mathcal H,\delta).
\end{equation}
\end{theorem}

\begin{proof}
See Appendix~\ref{app:uniform-hsic-application}. The proof maps FRHSIC to the uniform HSIC concentration theorem of \citet{ni2024uniform}: for each $h$, the empirical HSIC~\eqref{eq:hsic_empirical} at $\{(Z_i^h,S_i)\}_{i=1}^n$ is the biased estimator for the product kernel $\kappa_h\bigl((x,s),(x',s')\bigr):=k_\cz(h(x),h(x'))\,k_\cs(s,s')$; the sensitive coordinate is untransformed, so the singleton class on $\cs$ contributes zero Gaussian complexity; and the cited result applied to $\{\kappa_h:h\in\mathcal H\}$ yields~\eqref{eq:uniform-hsic-bound} with the constants of~\eqref{eq:Bn_explicit} (Proposition~\ref{prop:uniform_hsic}).
\end{proof}

\begin{corollary}[Population HSIC control for the learned encoder]
\label{cor:learned-encoder-hsic}
Suppose the assumptions of Theorem~\ref{thm:uniform-hsic-frhsic} hold, and let $\widehat h\in\mathcal H$ be any data-dependent encoder selected from the same sample, including one obtained by minimizing the FRHSIC objective~\eqref{eq:frhsic_objective}. Then, with probability at least $1-\delta$,
\[
    \hsic(\widehat h)\le\widehat{\hsic}_n(\widehat h)+B_n(\mathcal H,\delta).
\]
\end{corollary}

\begin{proof}
The event in Theorem~\ref{thm:uniform-hsic-frhsic} holds simultaneously for every $h\in\mathcal H$. Since $\widehat h\in\mathcal H$, evaluating~\eqref{eq:uniform-hsic-bound} at $h=\widehat h$ gives the claim.
\end{proof}

Corollary~\ref{cor:learned-encoder-hsic} is the train-to-population bridge: small empirical HSIC for the selected encoder implies small population dependence of the learned representation, up to the uniform concentration error, independently of the stochastic optimization used to minimize~\eqref{eq:frhsic_objective}.

\begin{corollary}[Root-$n$ rate under controlled encoder complexity]
\label{cor:uniform-hsic-rootn}
Under the assumptions of Theorem~\ref{thm:uniform-hsic-frhsic}, suppose $\mathcal G_n(\mathcal H)=O(n^{-1/2})$. Then, for every fixed $\delta\in(0,1)$,
\[
    \sup_{h\in\mathcal H}\bigl|\widehat{\hsic}_n(h)-\hsic(h)\bigr|=\widetilde O_p(n^{-1/2}),
\]
where $\widetilde O$ hides logarithmic factors, and the same rate controls the learned encoder $\widehat h$ through Corollary~\ref{cor:learned-encoder-hsic}.
\end{corollary}

\begin{proof}
Substituting $\mathcal G_n(\mathcal H)=O(n^{-1/2})$ into~\eqref{eq:Bn_explicit} gives $B_n(\mathcal H,\delta)=\widetilde O(n^{-1/2})$; the claim then follows from~\eqref{eq:uniform-hsic-bound} and Corollary~\ref{cor:learned-encoder-hsic}.
\end{proof}

\paragraph{Encoder classes covered by the rate}
The condition $\mathcal G_n(\mathcal H)=O(n^{-1/2})$ is a fixed-effective-complexity condition on the feasible encoder class. For bounded linear encoders $\mathcal H_{\mathrm{lin}}=\{h_W(x)=Wx:\|W\|_F\le R\}$ with $\|X\|\le M$ almost surely,
\[
    \widehat{\mathcal G}_n(\mathcal H_{\mathrm{lin}})
    \le
    R\,\bigl\|\tfrac1n\textstyle\sum_{i=1}^n\xi_iX_i^\top\bigr\|_F,
    \qquad\text{hence}\qquad
    \mathcal G_n(\mathcal H_{\mathrm{lin}})\le RM\sqrt{d_\cz/n},
\]
which is $O(n^{-1/2})$ for fixed representation dimension $d_\cz$. The same $n^{-1/2}$ order holds for fixed-depth, fixed-width MLP encoders with bounded weights, bounded inputs, and Lipschitz activations, and for finite-dimensional parametric families over compact parameter sets, under the corresponding complexity bounds of \citet{ni2024uniform}. The rate does not extend to arbitrary unregularized or growing-width networks.

\paragraph{Implementation}
Training optimizes \eqref{eq:frhsic_objective} with minibatch stochastic gradient methods: on a minibatch of size $m$ the HSIC penalty is the same trace formula~\eqref{eq:hsic_empirical} restricted to the batch, at $O(m^2)$ cost. Kernel bandwidths use the median heuristic. The penalty extends to Equal Opportunity by restricting the statistic to the $Y=1$ subset and to multiple sensitive attributes through a product kernel on the sensitive space, and $\lambda$ is chosen by held-out validation; details are in Appendix~\ref{sec:eo} and Appendix~\ref{sec:adaptive_lambda}.

%======================================================================
\section{Experimental Validation}
\label{sec:experiments}
%======================================================================

The experiments are consistent with the theory: the empirical estimator follows the predicted $O(n^{-1/2})$ rate, FRHSIC attains fairness--accuracy tradeoffs comparable to conditional-route baselines while keeping fairness stable across fresh downstream heads, and it trains about $36\times$ faster than FREM per epoch. We evaluate FRHSIC along four questions. \S\ref{sec:exp-synthetic} checks whether the empirical estimator follows the predicted convergence rate. \S\ref{sec:exp-real-tradeoff} compares its fairness--accuracy tradeoffs against the baselines on real data. \S\ref{sec:exp-transfer} tests whether the enforced fairness transfers to fresh downstream prediction heads. \S\ref{sec:exp-runtime} measures training time relative to conditional-route methods.

Code, cached results, and figures are available in a public repository at \url{https://github.com/Yijin911/FRHSIC}, with an archived DOI release accompanying the camera-ready version.

\subsection{Experimental setup}
\label{sec:exp-setup}

We use five real datasets with continuous sensitive attributes. Adult, ACS Income, MEPS, and COMPAS are classification tasks scored by accuracy; Crime is a regression task scored by mean squared error. Fairness is measured by held-out $\Delta_{\mathrm{GDP}}$, with lower values indicating less dependence on the sensitive attribute $S$.
\begin{itemize}
    \item \textbf{Adult}: $n \approx 30{,}000$ individuals; income ${>}50$K; sensitive attribute age.
    \item \textbf{ACS Income}: $n = 20{,}000$ from the 2018 California American Community Survey via folktables \citep{ding2021retiring}; income ${>}50$K; age.
    \item \textbf{MEPS}: $n \approx 13{,}000$ from the 2015 Medical Expenditure Panel Survey \citep{romano2020achieving}; healthcare utilization $\geq 10$ visits; age.
    \item \textbf{Crime} (the Communities and Crime dataset): $n \approx 2{,}000$ communities; violent crimes per population; racial composition.
    \item \textbf{COMPAS}: $n \approx 6{,}000$ defendants; two-year recidivism; age.
\end{itemize}
Features and $S$ are min--max scaled to $[0,1]$, fit on the training split only. We compare FRHSIC with the continuous-sensitive methods FREM \citep{kong2025fair} and Reg-GDP \citep{jiang2022generalized}, the adversarial method ADV \citep{grari2022fairness}, the binary-attribute methods MMD \citep{deka2023mmd} and LAFTR \citep{madras2018learning} applied to quartile-binned $S$, and an unconstrained baseline (Unfair, $\lambda=0$). Following \citet{kong2025fair}, all methods share a two-layer SELU encoder with hidden and representation dimension $50$ and a linear prediction head, train for $200$ epochs with Adam at learning rate $10^{-3}$, and are swept over $\lambda \in \{0.1, 1, 10, 100, 500\}$ with results averaged over five random $80/20$ splits, reported as mean $\pm$ standard deviation. Each method keeps the kernel bandwidth convention of its source publication; we do not retune bandwidths to favor any method, and the only fairness-specific hyperparameter FRHSIC tunes is $\lambda$, with kernel bandwidths fixed by the median heuristic; $\lambda$ is selected by the held-out procedure of Appendix~\ref{sec:adaptive_lambda}. Full preprocessing, hyperparameter grids, hardware, and the FREM $32$-anchor subsampling approximation are documented in Appendix~\ref{app:experiments}.

\subsection{Synthetic estimator behavior}
\label{sec:exp-synthetic}

Empirical HSIC converges to its population value at the predicted $O(n^{-1/2})$ rate, faster than the conditional-route EIPM estimator's $O(n^{-2/5})$. We draw $(Z,S)$ jointly Gaussian with correlation $\rho=0.5$ ($Z\sim\mathcal N(0,1)$, $S=\rho Z+\sqrt{1-\rho^2}\,\varepsilon$, $\varepsilon\sim\mathcal N(0,1)$), approximate each estimator's population value by its $n=20{,}000$ estimate, and report the absolute error at sample sizes from $50$ to $5{,}000$, averaged over $20$ replications with Gaussian kernels at bandwidth $1$. The EIPM estimator is the weighted MMD estimator of \citet{kong2025fair}, the same estimator used by the FREM baseline. On this dependence, $\widehat{\hsic}$ converges at the $O(n^{-1/2})$ rate that underlies the uniform train-to-population control of \S\ref{sec:uniform-train-pop}, while the EIPM estimator converges at the slower $O(n^{-2/5})$ rate. Figure~\ref{fig:convergence} reports both estimation errors against sample size. By Theorem~\ref{cor:hsic-mmd-equivalence}, HSIC and the conditional MMD integral bound each other up to an explicit spectral tail, so the joint estimator's faster $O(n^{-1/2})$ rate is a statistical-efficiency gain in estimating a closely related fairness quantity, equal to the conditional MMD integral up to that tail. The conditional-gap bound of Theorem~\ref{thm:hsic_gdp} is separately validated on a synthetic classification setup in Appendix~\ref{app:synthetic_validation}.

\begin{figure}[t]
    \centering
    \includegraphics[width=0.70\textwidth]{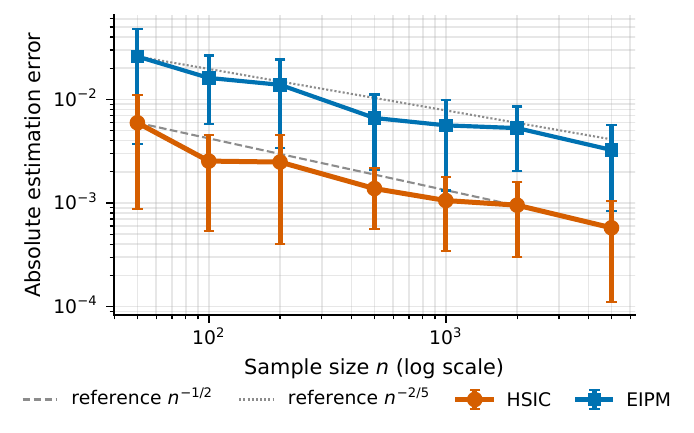}
    \caption{Synthetic estimator convergence. Absolute estimation error is plotted against sample size on log--log axes. Empirical HSIC follows the predicted $O(n^{-1/2})$ behavior and the smoothing-based EIPM estimator follows the slower $O(n^{-2/5})$ trend: the fitted log--log slopes are $-0.46$ for HSIC and $-0.44$ for EIPM, close to the predicted $-1/2$ and $-2/5$. The dashed and dotted gray lines are the $n^{-1/2}$ and $n^{-2/5}$ references.}
    \label{fig:convergence}
\end{figure}

\subsection{Fairness--accuracy tradeoffs}
\label{sec:exp-real-tradeoff}

On the five real datasets of \S\ref{sec:exp-setup}, scored by held-out $\Delta_{\mathrm{GDP}}$ against accuracy (mean squared error for Crime), FRHSIC achieves fairness--accuracy tradeoffs comparable to the strongest continuous-sensitive baselines, although it is not uniformly best on any dataset. Because HSIC and GDP penalties have different scales, equal $\lambda$ does not mean equal fairness, so Table~\ref{tab:real_data} compares at a matched operating point: for each method we report the lowest $\Delta_{\mathrm{GDP}}$ among models whose prediction performance stays within $1\%$ of the Unfair baseline, together with the bandwidth-free normalized area under the accuracy--fairness Pareto envelope, nAUP. Figures~\ref{fig:pareto} and~\ref{fig:pareto_crime} show the full frontiers over the $\lambda$ sweep, for the classification datasets and for Crime respectively.

No method dominates across datasets. By nAUP, FRHSIC is best on COMPAS, within $.004$ of the best on Crime, and mid-pack on Adult, ACS Income, and MEPS, staying within the spread of FREM, Reg-GDP, and the adversarial baselines. Among the binned baselines, LAFTR is consistently among the weakest by nAUP, while MMD on quartile-binned $S$ is competitive on Adult and ACS Income; binning $S$ thus yields no systematic advantage over the continuous-sensitive methods on these datasets. The matched band also removes a Reg-GDP artifact: Reg-GDP reaches near-zero $\Delta_{\mathrm{GDP}}$ only by collapsing accuracy outside the $1\%$ band, from $.849$ to $.758$ on Adult and from $.790$ to $.590$ on ACS Income at its unconstrained point. Differences within $\pm 0.01$ accuracy or $\pm 0.05$ on $\Delta_{\mathrm{GDP}}$ fall inside the five-split spread and should be read as practically equivalent. COMPAS age is integer-valued with few distinct values, so its centered sensitive Gram matrix has small rank, and the finite-sample bound of Theorem~\ref{thm:hsic_gdp} is governed by a few well-conditioned sensitive directions.

\begin{table}[!htbp]
\centering
\caption{Matched-performance comparison on the five real datasets. For each method we report the lowest $\Delta_{\mathrm{GDP}}$ among models whose performance stays within $1\%$ of the Unfair baseline. Perf is accuracy, except Crime where lower $\mathrm{MSE}\times10^2$ is better; $\Delta_{\mathrm{GDP}}$ lower is fairer; nAUP is the normalized area under the accuracy--fairness Pareto envelope, higher is better. Bold marks the best fairness method per column; ``---'' means no swept $\lambda$ meets the band. Setup and bandwidth conventions are in \S\ref{sec:exp-setup}.}
\label{tab:real_data}
\renewcommand{\arraystretch}{0.92}
\begin{tabular}{llccc}
\toprule
Dataset & Method & Perf.\ & $\Delta_{\mathrm{GDP}}$ ($\downarrow$) & nAUP ($\uparrow$) \\
\midrule
\multirow{7}{*}{Adult (Acc $\uparrow$)}
& Unfair              & $.849$ & $.749$ & --- \\
& FRHSIC (Ours)       & $.845$ & $.232$ & $.786$ \\
& FREM                & $.845$ & $.167$ & $.835$ \\
& Reg-GDP             & $.842$ & $\mathbf{.082}$ & $.872$ \\
& ADV                 & $.841$ & $.466$ & $.434$ \\
& MMD (binned)        & $.847$ & $.126$ & $\mathbf{.880}$ \\
& LAFTR (binned)      & $.842$ & $.385$ & $.555$ \\
\midrule
\multirow{7}{*}{ACS Income (Acc $\uparrow$)}
& Unfair              & $.790$ & $.514$ & --- \\
& FRHSIC (Ours)       & $.783$ & $.228$ & $.805$ \\
& FREM                & $.787$ & $.274$ & $.865$ \\
& Reg-GDP             & $.787$ & $.292$ & $.852$ \\
& ADV                 & $.785$ & $.458$ & $.713$ \\
& MMD (binned)        & $.783$ & $\mathbf{.157}$ & $\mathbf{.881}$ \\
& LAFTR (binned)      & $.787$ & $.474$ & $.754$ \\
\midrule
\multirow{7}{*}{MEPS (Acc $\uparrow$)}
& Unfair              & $.804$ & $.626$ & --- \\
& FRHSIC (Ours)       & $.797$ & $.329$ & $.464$ \\
& FREM                & $.800$ & $\mathbf{.317}$ & $\mathbf{.549}$ \\
& Reg-GDP             & $.801$ & $.370$ & $.469$ \\
& ADV                 & $.803$ & $.480$ & $.487$ \\
& MMD (binned)        & $.800$ & $.407$ & $.492$ \\
& LAFTR (binned)      & $.803$ & $.530$ & $.477$ \\
\midrule
\multirow{7}{*}{Crime (MSE $\downarrow$)}
& Unfair              & $1.91$ & $.029$ & --- \\
& FRHSIC (Ours)       & $1.91$ & $.029$ & $.686$ \\
& FREM                & $1.87$ & $\mathbf{.023}$ & $.680$ \\
& Reg-GDP             & $1.91$ & $.025$ & $\mathbf{.690}$ \\
& ADV                 & $1.92$ & $.029$ & $.002$ \\
& MMD (binned)        & ---    & ---   & $.573$ \\
& LAFTR (binned)      & $1.89$ & $.029$ & $.103$ \\
\midrule
\multirow{7}{*}{COMPAS (Acc $\uparrow$)}
& Unfair              & $.654$ & $.123$ & --- \\
& FRHSIC (Ours)       & $.649$ & $.090$ & $\mathbf{.679}$ \\
& FREM                & $.648$ & $.101$ & $.530$ \\
& Reg-GDP             & $.650$ & $\mathbf{.079}$ & $.480$ \\
& ADV                 & $.649$ & $.095$ & $.314$ \\
& MMD (binned)        & $.651$ & $.082$ & $.485$ \\
& LAFTR (binned)      & $.655$ & $.122$ & $.475$ \\
\bottomrule
\end{tabular}
\end{table}

\begin{figure}[htbp]
    \centering
    \includegraphics[width=0.8\textwidth]{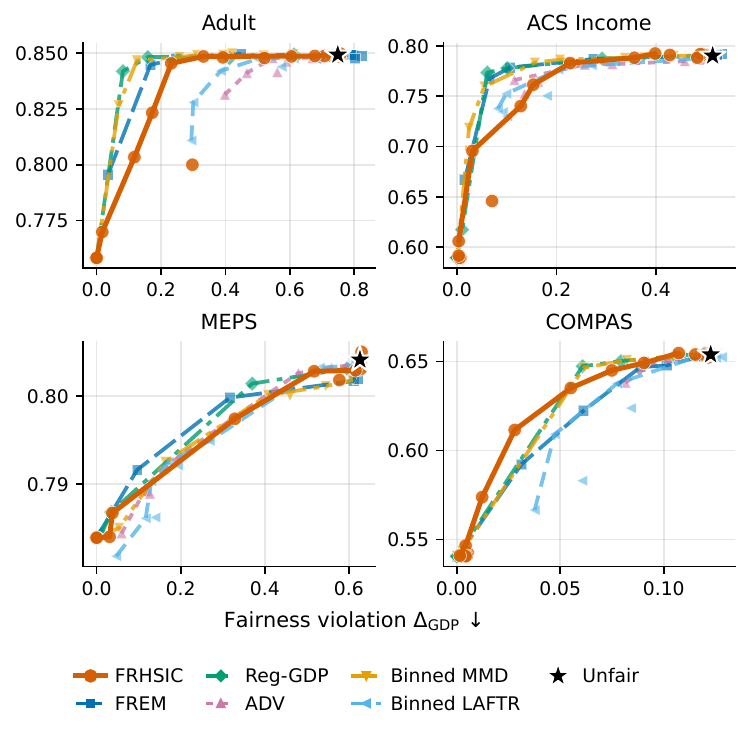}
    \caption{Fairness--accuracy frontiers on the classification datasets. Each point corresponds to one value of the fairness regularization parameter $\lambda$; lower $\Delta_{\mathrm{GDP}}$ is fairer and higher accuracy is better. Lines connect the Pareto-efficient points of each method. Axes are scaled separately by dataset to show the frontier structure. The black star marks the Unfair baseline, the unconstrained model trained without any fairness penalty ($\lambda=0$). FRHSIC is competitive with continuous-sensitive baselines across datasets, but no method uniformly dominates. Table~\ref{tab:real_data} gives the matched-performance summary.}
    \label{fig:pareto}
\end{figure}

\begin{figure}[t]
    \centering
    \includegraphics[width=0.55\textwidth]{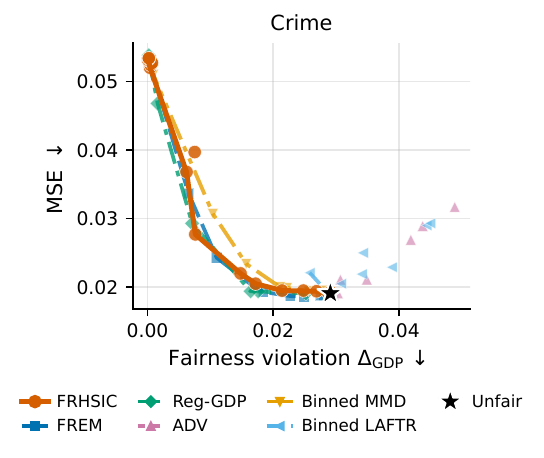}
    \caption{Fairness--prediction-error frontier on Crime. Each point corresponds to one value of $\lambda$; lower $\Delta_{\mathrm{GDP}}$ is fairer and lower MSE is better. Lines connect the Pareto-efficient points of each method; the black star marks the Unfair baseline ($\lambda=0$, no fairness penalty). The separate panel avoids mixing regression error with classification accuracy.}
    \label{fig:pareto_crime}
\end{figure}

\subsection{Transfer to fresh downstream heads}
\label{sec:exp-transfer}

Fairness enforced by FRHSIC stays stable across fresh downstream heads, consistent with a penalty that constrains all dependence between $Z$ and $S$ rather than the single head used in training. On Adult (at $\lambda=10$) and Crime (at $\lambda=100$), we freeze the learned representation of each method and train four fresh heads on it: linear, two-layer MLP, random forest, and SVM. Table~\ref{tab:transfer_summary} reports the standard deviation of $\Delta_{\mathrm{GDP}}$ across the four heads, a stability diagnostic in which lower values indicate fairness less tied to the head used during training. FRHSIC's cross-head variability is comparable to FREM and, on Adult, well below the Unfair baseline. Reg-GDP attains lower variability on Adult, but only at the accuracy collapse noted in \S\ref{sec:exp-real-tradeoff}. This is a stability diagnostic rather than a guarantee that every downstream head is fair; the full per-head breakdown is in Appendix~\ref{sec:transfer}.

\begin{table}[!htbp]
\centering
\caption{Cross-head $\Delta_{\mathrm{GDP}}$ standard deviation across four fresh heads (linear, MLP, random forest, SVM) trained on a single frozen representation at $\lambda = 10$ (Adult) and $\lambda = 100$ (Crime); lower indicates fairness less tied to the training head. Full per-head results are in Appendix~\ref{sec:transfer}.}
\label{tab:transfer_summary}
\renewcommand{\arraystretch}{0.95}
\small
\begin{tabular}{lcc}
\toprule
Method & Adult ($\lambda=10$) & Crime ($\lambda=100$) \\
\midrule
Unfair          & $0.348$ & --- \\
FRHSIC (Ours)   & $0.244$ & $0.001$ \\
FREM            & $0.247$ & $0.002$ \\
Reg-GDP         & $0.058$\,$^\dagger$ & $0.002$ \\
ADV             & $0.252$ & $0.005$ \\
\bottomrule
\end{tabular}
\par\smallskip
\footnotesize $^\dagger$ Reg-GDP attains lower cross-head variance on Adult but with accuracy collapsing to $0.758$ (Table~\ref{tab:real_data}).
\end{table}

\subsection{Computational efficiency}
\label{sec:exp-runtime}\label{sec:runtime}

FRHSIC trains substantially faster than FREM because it replaces FREM's per-sample weighted conditional-MMD computation with a single $O(m^2)$ HSIC trace per minibatch. Figure~\ref{fig:runtime} reports per-epoch wall-clock time, averaged over five epochs after a warm-up epoch, on a synthetic binary-classification task ($S\sim\mathrm{Unif}(0,1)$, features $X=(S+\varepsilon,\,\varepsilon')$ with $\varepsilon\sim\mathcal N(0,0.3^2)$ and $\varepsilon'\sim\mathcal N(0,1)$, label $Y\sim\mathrm{Bernoulli}(\sigma(X_1))$) of varying size, at batch size $256$ from $n=500$ to $n=20{,}000$, with all methods sharing the encoder and prediction head. At $n = 20{,}000$ and batch size $256$, FRHSIC completes an epoch in $2.82$s versus $100.7$s for FREM, a roughly $36\times$ speedup. Reg-GDP is faster still at $1.69$s because it evaluates conditional expectations on a fixed grid rather than kernel matrices, but it optimizes a head-specific smoothed moment criterion, whereas FRHSIC penalizes a representation-level distributional dependence that binds every downstream head. The full-batch $O(n^2)$ versus $O(n^3)$ analysis is given in Appendix~\ref{app:experiments}.

\begin{figure}[t]
    \centering
    \includegraphics[width=0.50\textwidth]{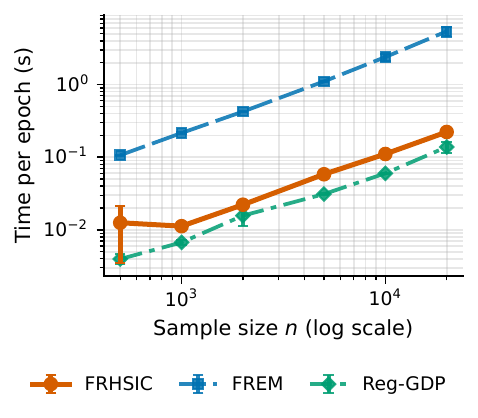}
    \caption{Per-epoch wall-clock time versus sample size at batch size $256$, on log--log axes; error bars are $\pm1$ standard deviation over repeats and are negligible except for FRHSIC at the smallest $n$. FRHSIC uses one HSIC trace per batch and trains about $36\times$ faster than FREM at $n=20{,}000$ (\S\ref{sec:runtime}). Reg-GDP is faster but optimizes a head-specific smoothed criterion rather than a representation-level distributional penalty.}
    \label{fig:runtime}
\end{figure}

Additional experiments are reported in Appendices~\ref{sec:transfer}\nobreakdash--\ref{sec:lambda_exp}: an empirical tightness sweep for the Theorem~\ref{thm:hsic_gdp} bound, a high-dimensional sensitive-attribute comparison, a mutual-information diagnostic, ablations over kernel choice and representation dimensionality, and validation of the $\lambda$-selection rule of Appendix~\ref{sec:adaptive_lambda}.

%======================================================================
\section{Related Work}\label{sec:related}

\paragraph{Fair representation learning}
\citet{zemel2013learning} introduced the framework of mapping data to a latent space that preserves task information while obfuscating $S$. Subsequent work extended this through variational autoencoders~\citep{louizos2015variational}, adversarial training~\citep{madras2018learning}, disentangled representations~\citep{creager2019flexibly}, MMD- and Sinkhorn-based formulations~\citep{NEURIPS2020_ac3870fc,deka2023mmd}, and data-domain transformations~\citep{quadrianto2019discovering}. Most of these methods are designed for categorical sensitive attributes and discretize continuous $S$ at training time.

\paragraph{Continuous-sensitive-attribute fairness criteria}
GDP~\citep{jiang2022generalized} is estimated via Nadaraya--Watson kernel smoothing on $S$. EIPM/FREM~\citep{kong2025fair} extends GDP to higher moments via a weighted empirical MMD with bandwidth $\gamma_n$ and leave-one-out correction. Kernel HGR / NLD~\citep{mary2019fairness,grari2022fairness} estimates a global dependence criterion via regularized kernel canonical correlation. Mutual information~\citep{cho2020fair} is estimated via density or copula approximations. Theorem~\ref{thm:joint-equals-integral} of \S\ref{sec:framework} disintegrates the joint-vs-product IPM into a $P_S$-weighted conditional contrast, and Corollary~\ref{cor:decomposable-class} recovers the conditional-integral IPM that GDP and EIPM instantiate when the witness class decomposes across $s$; kHGR is a supremum-of-correlations criterion outside this form.

\paragraph{HSIC and dependence measures}
HSIC was introduced by~\citet{gretton2005measuring} and has been used for independence testing~\citep{gretton2007kernel,albert2022adaptive}, feature selection~\citep{song2007supervised,song2012feature}, dimensionality reduction~\citep{ma2018nonlinear}, and as a training objective for deep networks~\citep{ma2020hsic}; its equivalence with distance covariance~\citep{sejdinovic2013equivalence} places it within a single family of joint-distribution dependence measures. HSIC has also been used directly as a fairness penalty: \citet{perezsuay2017fair} introduced it as the fairness term in kernel regression and dimensionality reduction, \citet{li2022kernel} added closed-form ridge solutions and a Gaussian-process treatment, and \citet{quadrianto2019discovering} used it for independence from the sensitive attribute in a data-domain model. Our contribution is the continuous-$S$ disintegration that ties the HSIC penalty to the conditional-integral criteria (Theorem~\ref{thm:joint-equals-integral}, Corollary~\ref{cor:decomposable-class}) and the spectral and train-to-population guarantees (Theorems~\ref{cor:hsic-mmd-equivalence} and~\ref{thm:uniform-hsic-frhsic}) that these empirical and ridge-based treatments lack.

%======================================================================
\section{Discussion}\label{sec:discussion}

\paragraph{The joint-distribution view}
The joint-distribution view reaches the same independence target as the conditional-integral criteria but changes the object that is estimated. Rather than a conditional law at each sensitive value, it estimates a single closed-form $V$-statistic, which controls smoothed conditional deviations and, under spectral regularity, the conditional MMD integral up to a spectral tail (Theorems~\ref{thm:hsic_smoothed_conditional_control} and~\ref{cor:hsic-mmd-equivalence}).

\paragraph{Extensions}
The framework extends to representation-level Equal Opportunity by computing the joint discrepancy on the $Y=1$ subset, and to multiple continuous sensitive attributes by taking a product kernel on $\mathcal S_1\times\mathcal S_2$. Both extensions are derived in Appendix~\ref{sec:eo}. Replacing the reference distribution $P_Z\otimes P_S$ by $P_Z\otimes\mathrm{Unif}(\mathcal S)$ yields a minority-protection target in which rare values of $S$ receive equal weight; this is a modeling choice and does not affect the structural identity of Theorem~\ref{thm:joint-equals-integral}.

\paragraph{Open questions}
Two questions remain. First, the empirical bound in Theorem~\ref{thm:hsic_gdp} scales as $\|f\|_{\mathcal F_Z}^2/\hat\lambda_m$; tighter spectral arguments exploiting the full singular-value distribution of $\widetilde L$ would yield sharper guarantees on the integral form. Second, the framework treats the joint discrepancy with an RKHS witness class; non-RKHS instances (Wasserstein, total variation) inherit the structural identity but lose the $O(n^{-1/2})$ rate and closed-form estimator. We leave these to future work.

%======================================================================
% Appendices: the journal Supplementary Materials, inlined for arXiv.
%======================================================================
\appendix

\section{Proofs}
\label{app:proofs}
%======================================================================

This section collects the proofs of the results stated in the main text, in the order in which the results appear there. The technical lemma used in the proof of Proposition~\ref{thm:EIPM_HSIC} (Lemma~\ref{lem:fubini}) is recorded in Appendix~\ref{app:notation} below.

\subsection{Proof of Proposition~\ref{thm:subgroup_equivalence}: conditions (1)--(4)}
\label{app:pf_thm_subgroup_equivalence}

\begin{proof}
We show $(i) \Rightarrow (ii) \Rightarrow (iii) \Rightarrow (iv) \Rightarrow (i)$.

\emph{$(i) \Rightarrow (ii)$.} Since $\mathcal F$ is characteristic, $\mathrm{IPM}_{\mathcal F}(P_{Z,S}, P_Z \otimes P_S) = 0$ implies $P_{Z, S} = P_Z \otimes P_S$.

\emph{$(ii) \Rightarrow (iii)$.} On the Polish product space $\mathcal{Z} \times \mathcal{S}$, the disintegration theorem \citep[Thm.~5.4]{kallenberg2002foundations} yields a regular conditional probability $\{P_{Z|S=s}\}_{s \in \mathcal{S}}$ such that for every Borel $A \subseteq \mathcal{Z}, B \subseteq \mathcal{S}$,
\[
    P_{Z, S}(A \times B) = \int_B P_{Z|S=s}(A)\,dP_S(s).
\]
Under $(ii)$, the same Borel set $A \times B$ also satisfies $P_{Z, S}(A \times B) = P_Z(A)\,P_S(B) = \int_B P_Z(A)\,dP_S(s)$. Comparing the two expressions, $\int_B [P_{Z|S=s}(A) - P_Z(A)]\,dP_S(s) = 0$ for every Borel $B$, so $P_{Z|S=s}(A) = P_Z(A)$ for $P_S$-a.e.\ $s$. Choose a countable algebra $\mathcal{A}$ generating the Borel $\sigma$-algebra of $\mathcal{Z}$ (which exists since $\mathcal{Z}$ is Polish); the exceptional $P_S$-null sets, taken over the countable family $\mathcal{A}$, have countable union of $P_S$-measure zero. Hence $P_{Z|S=s}|_{\mathcal{A}} = P_Z|_{\mathcal{A}}$ for $P_S$-a.e.\ $s$, and by the $\pi$-$\lambda$ theorem $P_{Z|S=s} = P_Z$ on the full Borel $\sigma$-algebra for $P_S$-a.e.\ $s$. Equivalence of ``some'' and ``every'' regular conditional version follows from the $P_S$-a.e.\ uniqueness clause of the disintegration theorem.

\emph{$(iii) \Rightarrow (iv)$.} For every bounded measurable $f$ and $P_S$-a.e.\ $s$,
\[
    \mathbb{E}[f(Z) \mid S = s] = \int f(z)\,dP_{Z|S=s}(z) = \int f(z)\,dP_Z(z) = \mathbb{E}[f(Z)],
\]
where the second equality uses $(iii)$.

\emph{$(iv) \Rightarrow (i)$.} Indicator functions of Borel sets are bounded and measurable, so applying $(iv)$ with $f = \mathbf{1}_A$ gives $P_{Z|S=s}(A) = P_Z(A)$ for $P_S$-a.e.\ $s$, for every Borel $A \subseteq \mathcal{Z}$. Take a countable algebra $\mathcal{A}$ generating the Borel $\sigma$-algebra of the Polish space $\mathcal{Z}$; the exceptional $P_S$-null sets over the countable family $\mathcal{A}$ have null union, so $P_{Z|S=s}|_{\mathcal{A}} = P_Z|_{\mathcal{A}}$ for $P_S$-a.e.\ $s$, and the $\pi$-$\lambda$ theorem gives $P_{Z|S=s} = P_Z$ for $P_S$-a.e.\ $s$. Hence $Z \perp S$, i.e., $P_{Z,S} = P_Z \otimes P_S$, so $\mathrm{IPM}_{\mathcal F}(P_{Z,S}, P_Z \otimes P_S) = 0$.
\end{proof}

\subsection{Proof of Proposition~\ref{thm:subgroup_equivalence}: the MMD-based EIPM instance}
\label{app:pf_prop_equivalence}

\begin{proof}
($\Rightarrow$) If $\hsic(Z, S) = 0$, then $\mu_{Z,S} = \mu_Z \otimes \mu_S$ in $\cf_{\cz \otimes \cs}$. Because $k_\cz$ and $k_\cs$ are characteristic, the product kernel $k_\cz \otimes k_\cs$ is $\mathcal{I}$-characteristic, i.e., its joint mean embedding distinguishes $P_{Z,S}$ from $P_Z \otimes P_S$ (Theorem~3 of \citealp{szabo2018characteristic}, taking $M=2$). Hence $\mu_{Z,S} = \mu_Z \otimes \mu_S$ implies $P_{Z,S} = P_Z \otimes P_S$, i.e., $Z \perp S$.
Hence $P_{Z \mid S = s} = P_Z$ for $P_S$-almost every $s$, so $\mathrm{MMD}(P_{Z \mid S = s}, P_Z) = 0$ a.s., and $\mathrm{EIPM}_\mathcal{V}(Z; S) = \be_S[\mathrm{MMD}(P_{Z \mid S}, P_Z)] = 0$.

($\Leftarrow$) Suppose $\mathrm{EIPM}_\mathcal{V}(Z; S) = 0$.
Since $\mathrm{MMD}(P_{Z \mid S = s}, P_Z) \geq 0$, this implies $\mathrm{MMD}(P_{Z \mid S = s}, P_Z) = 0$ for $P_S$-almost every $s$.
Because $k_\cz$ is characteristic, MMD is a metric on probability distributions, so $P_{Z \mid S = s} = P_Z$ a.s., i.e., $Z \perp S$.
Independence then implies that the joint mean embedding factorizes, $\mu_{Z,S} = \mu_Z \otimes \mu_S$, so $\hsic(Z, S) = \|\mu_{Z,S} - \mu_Z \otimes \mu_S\|^2_{\cf_{\cz \otimes \cs}} = 0$.
\end{proof}

\subsection{Proof of Proposition~\ref{thm:subgroup_equivalence}: downstream-head (GDP) form}
\label{app:pf_cor_hsic_gdp_population}

Since $k_\cz\otimes k_\cs$ is characteristic, $\mathrm{HSIC}(Z,S)=0$ is the product-RKHS instance of condition (1) of Proposition~\ref{thm:subgroup_equivalence}, so $P_{Z\mid S=s}=P_Z$ for $P_S$-almost every $s$ by condition (3), hence $\mathbb E[f(Z)\mid S=s]=\mathbb E[f(Z)]$ for $P_S$-almost every $s$ and every measurable $f:\cz\to\mathbb R$ with $\mathbb E|f(Z)|<\infty$. Therefore $\Delta_{\mathrm{GDP}}(f)=\mathbb E_S\bigl|\mathbb E[f(Z)\mid S]-\mathbb E[f(Z)]\bigr|=0$.

\subsection{Proof of Theorem~\ref{thm:joint-equals-integral} (Disintegration of the joint-vs-product IPM)}
\label{app:pf_thm_joint_disintegration}

Because $\cz$ and $\cs$ are standard Borel, the disintegration theorem \citep[Theorem~5.4]{kallenberg2002foundations} gives $P_{Z,S}(dz\,ds)=P_{Z\mid S=s}(dz)\,P_S(ds)$, and by construction $(P_Z\otimes P_S)(dz\,ds)=P_Z(dz)\,P_S(ds)$. Fix $f\in\mathcal F$ with $\|f\|_\infty\le b<\infty$. Both $P_{Z\mid S=s}(dz)\,P_S(ds)$ and $P_Z(dz)\,P_S(ds)$ are probability measures and $|f|\le b$, so each iterated integral below is absolutely convergent and the Fubini--Tonelli theorem applies:
\[
\mathbb E_{P_{Z,S}}[f]-\mathbb E_{P_Z\otimes P_S}[f]
=\int_{\cs}\!\left(\int_{\cz} f(z,s)\,P_{Z\mid S=s}(dz)-\int_{\cz} f(z,s)\,P_Z(dz)\right)P_S(ds).
\]
The inner difference is $\int_{\cz} f(z,s)\,d(P_{Z\mid S=s}-P_Z)(z)$. Taking absolute values and the supremum over $f\in\mathcal F$ yields the identity.

\subsection{Proof of Corollary~\ref{cor:decomposable-class} (Recovery of the conditional-integral IPM)}
\label{app:pf_cor_decomposable}

By Theorem~\ref{thm:joint-equals-integral},
\[
\mathrm{IPM}_{\mathcal F}(P_{Z,S},P_Z\otimes P_S)=\sup_{f\in\mathcal F}\Bigl|\,\mathbb E_S\!\int_{\cz} f(z,S)\,d(P_{Z\mid S}-P_Z)(z)\Bigr|.
\]
\emph{Upper bound.} Fix $f\in\mathcal F$. By Definition~\ref{def:decomposable-class}(2), $f(\cdot,s)\in\mathcal V$ for $P_S$-a.e.\ $s$, so for those $s$
\[
\Bigl|\int_{\cz} f(z,s)\,d(P_{Z\mid S=s}-P_Z)(z)\Bigr|\le \sup_{g\in\mathcal V}\Bigl|\int_{\cz} g\,d(P_{Z\mid S=s}-P_Z)\Bigr|=d_{\mathcal V}(P_{Z\mid S=s},P_Z).
\]
By the triangle inequality for the $P_S$-integral, $\bigl|\mathbb E_S\!\int f(z,S)\,d(P_{Z\mid S}-P_Z)\bigr|\le \mathbb E_S\bigl[d_{\mathcal V}(P_{Z\mid S},P_Z)\bigr]$; taking the supremum over $f\in\mathcal F$ gives $\mathrm{IPM}_{\mathcal F}\le \mathbb E_S[d_{\mathcal V}(P_{Z\mid S},P_Z)]$.
\emph{Lower bound.} Fix $\varepsilon>0$. Since $\mathcal V$ is symmetric ($g\in\mathcal V\Rightarrow-g\in\mathcal V$), $d_{\mathcal V}(P_{Z\mid S=s},P_Z)=\sup_{g\in\mathcal V}\bigl(\int g\,dP_{Z\mid S=s}-\int g\,dP_Z\bigr)$, the absolute value being attained because $-g\in\mathcal V$. By the selector assumption of Corollary~\ref{cor:decomposable-class} there is a measurable map $s\mapsto g_s^\varepsilon\in\mathcal V$ with $\int g_s^\varepsilon\,d(P_{Z\mid S=s}-P_Z)\ge d_{\mathcal V}(P_{Z\mid S=s},P_Z)-\varepsilon$ for $P_S$-a.e.\ $s$. By Definition~\ref{def:decomposable-class}(1), $f_\star(z,s):=g_s^\varepsilon(z)$ lies in $\mathcal F$, and since the integrand is nonnegative up to $\varepsilon$,
\[
\mathrm{IPM}_{\mathcal F}\ge \Bigl|\mathbb E_S\!\int g_S^\varepsilon\,d(P_{Z\mid S}-P_Z)\Bigr|\ge \mathbb E_S\bigl[d_{\mathcal V}(P_{Z\mid S},P_Z)\bigr]-\varepsilon.
\]
Letting $\varepsilon\downarrow0$ gives $\mathrm{IPM}_{\mathcal F}\ge\mathbb E_S[d_{\mathcal V}(P_{Z\mid S},P_Z)]$. Combining the two bounds and recalling $\mathcal I_{d_{\mathcal V}}(Z;S)=\mathbb E_S[d_{\mathcal V}(P_{Z\mid S},P_Z)]$ from the preliminaries yields the identity.

\subsection{Proof of Proposition~\ref{thm:EIPM_HSIC} (HSIC and Expected MMD)}
\label{app:pf_thm_eipm_hsic}

\begin{proof}
\textbf{Step 1: Decomposition.}
By definition of the joint mean embedding and the tower property of conditional expectation,
\[
    \mu_{Z,S} = \mathbb{E}_{Z,S}\!\left[k_\cz(Z, \cdot) \otimes k_\cs(S, \cdot)\right] = \mathbb{E}_S\!\left[\mu_{Z\mid S} \otimes k_\cs(S, \cdot)\right],
\]
and by the definition $\mu_S = \mathbb{E}_{S'}[k_\cs(S', \cdot)]$ and linearity of the tensor product (with $\mu_Z$ constant in $S'$),
$\mu_Z \otimes \mu_S = \mathbb{E}_{S'}[\mu_Z \otimes k_\cs(S', \cdot)].$
Subtracting and using linearity of the Bochner integral (Lemma~\ref{lem:fubini}),
\begin{equation}
\label{eq:joint_minus_marginal}
    \mu_{Z,S} - \mu_Z \otimes \mu_S = \mathbb{E}_S\!\left[(\mu_{Z\mid S} - \mu_Z) \otimes k_\cs(S, \cdot)\right].
\end{equation}
Squaring the norm and writing it as an inner product gives, with $D_S := (\mu_{Z\mid S} - \mu_Z) \otimes k_\cs(S, \cdot)$,
\begin{align*}
    \hsic(Z, S)
    &= \langle \mathbb{E}_S[D_S],\, \mathbb{E}_{S'}[D_{S'}] \rangle_{\cf_{\cz \otimes \cs}} \\
    &= \mathbb{E}_{S, S'}\!\left[\langle D_S,\, D_{S'}\rangle_{\cf_{\cz \otimes \cs}}\right],
\end{align*}
where the exchange of expectation and inner product is justified by Lemma~\ref{lem:fubini} together with the bound $\|(\mu_{Z\mid S} - \mu_Z) \otimes k_\cs(S, \cdot)\|_{\cf_{\cz \otimes \cs}} \leq 2\sqrt{\kappa_\cz \kappa_\cs}$.
By the tensor-product inner product identity $\langle a \otimes b, a' \otimes b'\rangle = \langle a, a'\rangle \langle b, b'\rangle$ and the reproducing property $\langle k_\cs(S, \cdot), k_\cs(S', \cdot)\rangle_{\cf_\cs} = k_\cs(S, S')$,
\[
    \hsic(Z, S) = \mathbb{E}_{S, S'}\!\left[\langle \mu_{Z\mid S} - \mu_Z, \mu_{Z\mid S'} - \mu_Z\rangle_{\cf_\cz} \cdot k_\cs(S, S')\right],
\]
which establishes~\eqref{eq:hsic_mmd}.

\textbf{Step 2: Squared expected MMD bound.}
Starting from~\eqref{eq:joint_minus_marginal}, apply the triangle inequality (Jensen's inequality for the convex norm):
\begin{align*}
    \hsic(Z, S)^{1/2}
    &= \left\| \mathbb{E}_S[(\mu_{Z\mid S} - \mu_Z) \otimes k_\cs(S, \cdot)] \right\|_{\cf_{\cz \otimes \cs}}\\
    &\leq \mathbb{E}_S\!\left[\|(\mu_{Z\mid S} - \mu_Z) \otimes k_\cs(S, \cdot)\|_{\cf_{\cz \otimes \cs}}\right].
\end{align*}
The tensor-product norm factorizes: $\|a \otimes b\|_{\cf_{\cz \otimes \cs}} = \|a\|_{\cf_\cz} \|b\|_{\cf_\cs}$, and $\|k_\cs(S, \cdot)\|_{\cf_\cs} = \sqrt{k_\cs(S, S)} \leq \sqrt{\kappa_\cs}$.
Therefore
\[
    \hsic(Z, S)^{1/2} \leq \mathbb{E}_S\!\left[\|\mu_{Z\mid S} - \mu_Z\|_{\cf_\cz} \sqrt{k_\cs(S, S)}\right] \leq \sqrt{\kappa_\cs} \cdot \mathbb{E}_S\!\left[\mathrm{MMD}_{k_\cz}(P_{Z\mid S},P_Z)\right].
\]
Squaring both sides yields~\eqref{eq:hsic_conditional}.
\end{proof}

\subsection{Proof of Theorem~\ref{thm:hsic_smoothed_conditional_control} (Smoothed conditional-distribution control by HSIC)}
\label{app:pf_prop_hsic_smooth_control}

\begin{proof}
Since $k_\cz$ is bounded, $\|\mu_{Z\mid S=s}\|_{\cf_\cz}\le \sqrt{\kappa_\cz}$ and $\|\mu_Z\|_{\cf_\cz}\le \sqrt{\kappa_\cz}$, so $\Delta\in L^2(P_S;\cf_\cz)$. The joint mean embedding satisfies
\[
\mu_{Z,S}-\mu_Z\otimes\mu_S
=
\mathbb E_S[(\mu_{Z\mid S=S}-\mu_Z)\otimes k_\cs(S,\cdot)]
=
\mathbb E_S[\Delta(S)\otimes k_\cs(S,\cdot)].
\]
Taking the squared norm in $\cf_{\cz\otimes\cs}$ gives the first identity.

For $g\in\cf_\cs$, the partial contraction of the tensor against $g$ over the $\cf_\cs$ factor gives an element of $\cf_\cz$:
\[
\bigl(\mathrm{id}_{\cf_\cz}\otimes\langle\,\cdot\,,g\rangle_{\cf_\cs}\bigr)
\bigl(\mathbb E_S[\Delta(S)\otimes k_\cs(S,\cdot)]\bigr)
=
\mathbb E_S[g(S)\,\Delta(S)].
\]
Hence
\[
\bigl\|\mathbb E_S[g(S)\Delta(S)]\bigr\|_{\cf_\cz}
\le
\|g\|_{\cf_\cs}\,
\bigl\|\mathbb E_S[\Delta(S)\otimes k_\cs(S,\cdot)]\bigr\|_{\cf_{\cz\otimes\cs}},
\]
which, with the first identity, proves~\eqref{eq:hsic_controls_weighted_embedding}. Two equivalent reformulations follow. Writing $m_f(s):=\mathbb E[f(Z)\mid S=s]-\mathbb E[f(Z)]=\langle f,\Delta(s)\rangle_{\cf_\cz}$ for $f\in\cf_\cz$, the same contraction argument applied to the $\cf_\cz$-coordinate gives the fixed-head form
\begin{equation}
\label{eq:hsic_controls_smoothed_head_gap}
    \bigl\|\mathbb E_S[m_f(S)\,k_\cs(S,\cdot)]\bigr\|_{\cf_\cs}^2
    \le
    \|f\|_{\cf_\cz}^2\,\hsic(Z,S).
\end{equation}
Finally, since $\mathbb E_S[g(S)m_f(S)]=\bigl\langle f,\;\mathbb E_S[g(S)\Delta(S)]\bigr\rangle_{\cf_\cz}$, Cauchy--Schwarz together with~\eqref{eq:hsic_controls_weighted_embedding} gives the bilinear form
\begin{equation}
\label{eq:hsic_controls_covariance_contrast}
    \bigl|\mathbb E_S[g(S)m_f(S)]\bigr|^2
    \le
    \|f\|_{\cf_\cz}^2\,\|g\|_{\cf_\cs}^2\,\hsic(Z,S)
\end{equation}
for all $f\in\cf_\cz$ and $g\in\cf_\cs$.
\end{proof}

\subsection{Proof of Corollary~\ref{cor:localized_sensitive_contrasts} (Localized sensitive-value contrasts)}
\label{app:pf_cor_localized}

\begin{proof}
Apply Inequality~\eqref{eq:hsic_controls_weighted_embedding} with $g=k_\cs(\cdot,s_0)\in\cf_\cs$: by the reproducing property $g(S)=k_\cs(S,s_0)$ and $\|g\|_{\cf_\cs}^2=\langle k_\cs(\cdot,s_0),k_\cs(\cdot,s_0)\rangle_{\cf_\cs}=k_\cs(s_0,s_0)$, so $\bigl\|\mathbb E_S[k_\cs(S,s_0)\Delta(S)]\bigr\|_{\cf_\cz}^2\le k_\cs(s_0,s_0)\,\hsic(Z,S)$, which is the single-point bound since $\Delta(S)=\mu_{Z\mid S}-\mu_Z$. The two-point form follows by the same argument with $g=k_\cs(\cdot,s_0)-k_\cs(\cdot,s_1)\in\cf_\cs$, for which $g(S)=k_\cs(S,s_0)-k_\cs(S,s_1)$ and $\|g\|_{\cf_\cs}^2=\|k_\cs(s_0,\cdot)-k_\cs(s_1,\cdot)\|_{\cf_\cs}^2$.
\end{proof}

\subsection{Proof of Theorem~\ref{cor:hsic-mmd-equivalence} (upper bound: spectral control of the conditional MMD integral)}
\label{app:pf_cor_spectral_conditional_mmd}

\begin{proof}
Since $k_\cs$ is bounded, $T_S$ is a self-adjoint, positive, trace-class operator on $L^2(P_S)$; let $(\lambda_\ell,\psi_\ell)_{\ell\ge1}$ be its eigensystem with $\lambda_\ell\ge0$ and $\{\psi_\ell\}$ orthonormal in $L^2(P_S)$, and recall the Mercer expansion $k_\cs(s,t)=\sum_{\ell\ge1}\lambda_\ell\psi_\ell(s)\psi_\ell(t)$ in $L^2(P_S\otimes P_S)$. By Theorem~\ref{thm:hsic_smoothed_conditional_control},
\[
\hsic(Z,S)=\bigl\|\mathbb E_S[\Delta(S)\otimes k_\cs(S,\cdot)]\bigr\|_{\cf_{\cz\otimes\cs}}^2 .
\]
Substituting the Mercer expansion and writing $\Delta_\ell:=\mathbb E_S[\Delta(S)\psi_\ell(S)]=\int\Delta(s)\psi_\ell(s)\,dP_S(s)\in\cf_\cz$ gives, by orthonormality of $\{\psi_\ell\}$,
\[
\mathbb E_S[\Delta(S)\otimes k_\cs(S,\cdot)]
=\sum_{\ell\ge1}\lambda_\ell\,\Delta_\ell\otimes\psi_\ell,
\qquad
\hsic(Z,S)=\sum_{\ell\ge1}\lambda_\ell\|\Delta_\ell\|_{\cf_\cz}^2,
\]
which is the first identity. For any $m$ with $\lambda_m>0$, since $\lambda_\ell\ge\lambda_m$ for $\ell\le m$ and every term is nonnegative,
\[
\sum_{\ell=1}^m\|\Delta_\ell\|_{\cf_\cz}^2
\le\frac{1}{\lambda_m}\sum_{\ell=1}^m\lambda_\ell\|\Delta_\ell\|_{\cf_\cz}^2
\le\frac{1}{\lambda_m}\hsic(Z,S).
\]
Because $\{\psi_\ell\}$ is orthonormal in $L^2(P_S)$, the $\cf_\cz$-valued coefficients of $\Delta$ are $\Delta_\ell$, so $\|P_m\Delta\|_{L^2(P_S;\cf_\cz)}^2=\sum_{\ell=1}^m\|\Delta_\ell\|_{\cf_\cz}^2$, giving the projected bound. Finally, $P_m$ is an orthogonal projection on $L^2(P_S;\cf_\cz)$, so $\|\Delta\|^2=\|P_m\Delta\|^2+\|(I-P_m)\Delta\|^2\le \lambda_m^{-1}\hsic(Z,S)+\rho_m^2$; the identity $\|\Delta\|_{L^2(P_S;\cf_\cz)}^2=\mathbb E_S[\|\Delta(S)\|_{\cf_\cz}^2]=\mathbb E_S[\mathrm{MMD}_{k_\cz}^2(P_{Z\mid S},P_Z)]$ uses $\|\Delta(s)\|_{\cf_\cz}=\mathrm{MMD}_{k_\cz}(P_{Z\mid S=s},P_Z)$.
\end{proof}

\subsection{Proof of Corollary~\ref{cor:spectral_gdp_control} (Spectral control of GDP for RKHS heads)}
\label{app:pf_cor_spectral_gdp}

\begin{proof}
By the reproducing property, $m_f(s)=\mathbb E[f(Z)\mid S=s]-\mathbb E[f(Z)]=\langle f,\Delta(s)\rangle_{\cf_\cz}$, so $m_f\in L^2(P_S)$ with $L^2(P_S)$ coefficients $\int m_f(s)\psi_\ell(s)\,dP_S(s)=\langle f,\Delta_\ell\rangle_{\cf_\cz}$, where $\Delta_\ell:=\mathbb E_S[\Delta(S)\psi_\ell(S)]$ as in the proof of Theorem~\ref{cor:hsic-mmd-equivalence}. Hence, by Cauchy--Schwarz and that theorem,
\[
\|P_m m_f\|_{L^2(P_S)}^2
=\sum_{\ell=1}^m \langle f,\Delta_\ell\rangle_{\cf_\cz}^2
\le \|f\|_{\cf_\cz}^2\sum_{\ell=1}^m\|\Delta_\ell\|_{\cf_\cz}^2
\le \frac{\|f\|_{\cf_\cz}^2}{\lambda_m}\,\hsic(Z,S).
\]
Since $P_m$ is an orthogonal projection on $L^2(P_S)$, $\|m_f\|_{L^2(P_S)}^2=\|P_m m_f\|^2+\|(I-P_m)m_f\|^2\le \lambda_m^{-1}\|f\|_{\cf_\cz}^2\hsic(Z,S)+r_{m,f}^2$. Finally, by Jensen's inequality $\Delta_{\mathrm{GDP}}(f)=\mathbb E_S|m_f(S)|\le (\mathbb E_S m_f(S)^2)^{1/2}=\|m_f\|_{L^2(P_S)}$, so $\Delta_{\mathrm{GDP}}(f)^2\le \lambda_m^{-1}\|f\|_{\cf_\cz}^2\hsic(Z,S)+r_{m,f}^2$.
\end{proof}

\subsection{Proof of Theorem~\ref{thm:hsic_gdp} (Empirical demographic-parity control by HSIC)}
\label{app:pf_thm_hsic_gdp}

\begin{proof}[Proof of Theorem~\ref{thm:hsic_gdp}]
{\sloppy
We work with empirical quantities throughout. Denote by $\hat d_{S_i} := k_\cz(Z_i, \cdot) - \hat\mu_Z$ the empirical centered embedding at $S = S_i$, with $\hat\mu_Z = n^{-1}\sum_{j=1}^n k_\cz(Z_j, \cdot)$, and write $\bm{\delta}_f := (\hat\delta_{f, 1}, \ldots, \hat\delta_{f, n})^\top$ for the centered prediction-deviation vector.\par}

\textbf{Step 1: Spectral identity for $\widehat{\hsic}$ using positive eigenpairs only.}
Let $K \in \mathbb{R}^{n \times n}$ have entries $K_{ij} = k_\cz(Z_i, Z_j)$ and $\widetilde K = H K H$. The biased V-statistic estimator is
\[
    \widehat{\hsic}(Z, S) = \frac{1}{n^2} \mathrm{tr}(K H L H) = \frac{1}{n^2}\mathrm{tr}(\widetilde K \, \widetilde L).
\]
Since $L$ is symmetric positive semi-definite and $H = H^\top$ with $H^2 = H$, $\widetilde L = HLH$ is symmetric PSD. Let $\hat\lambda_1 \geq \cdots \geq \hat\lambda_r > 0$ be the positive eigenvalues of $n^{-1}\widetilde L$ with associated orthonormal eigenvectors $\hat e_1, \ldots, \hat e_r \in \mathbb{R}^n$, where $r = \mathrm{rank}(\widetilde L)$. Then
\[
    \frac{1}{n}\widetilde L = \sum_{j = 1}^r \hat\lambda_j\, \hat e_j \hat e_j^\top, \qquad P_m = \sum_{j=1}^m \hat e_j \hat e_j^\top.
\]
Substituting and using $\widehat{\hsic} = n^{-1}\mathrm{tr}(\widetilde K \cdot n^{-1}\widetilde L)$,
\begin{equation}
\label{eq:hsic_spectral}
    \widehat{\hsic}(Z, S) = \frac{1}{n}\sum_{j = 1}^r \hat\lambda_j \cdot \hat e_j^\top \widetilde K \hat e_j.
\end{equation}

\textbf{Step 2: Lower bound via the $m$-th eigenvalue.}
Since $\widetilde K = HKH$ is PSD, every term in~\eqref{eq:hsic_spectral} is nonnegative; dropping the terms $j>m$ and using $\hat\lambda_j \geq \hat\lambda_m$ for $j \leq m$,
\begin{equation}
\label{eq:hsic_lower}
    \widehat{\hsic}(Z, S) \;\geq\; \frac{1}{n}\sum_{j=1}^m \hat\lambda_j\, \hat e_j^\top \widetilde K \hat e_j \;\geq\; \frac{\hat\lambda_m}{n}\sum_{j=1}^m \hat e_j^\top \widetilde K \hat e_j \;=\; \frac{\hat\lambda_m}{n}\,\mathrm{tr}\!\big(P_m\, \widetilde K\, P_m\big),
\end{equation}
where we used $\sum_{j=1}^m \hat e_j^\top \widetilde K \hat e_j = \mathrm{tr}(P_m \widetilde K P_m)$ (since $P_m = \sum_{j=1}^m \hat e_j \hat e_j^\top$, $P_m^2 = P_m$, and the cyclic-trace identity).

\textbf{Step 3: Bound the projected prediction deviation by $\mathrm{tr}(P_m\widetilde K P_m)$.}
By the reproducing property, $\hat\delta_{f, i} = \langle f, \hat d_{S_i}\rangle_{\cf_\cz}$, so the vector $\bm{\delta}_f$ has entries $(\bm{\delta}_f)_i = \langle f, \hat d_{S_i}\rangle_{\cf_\cz}$. For any unit vector $u \in \mathbb{R}^n$, the linear combination $\sum_i u_i \hat d_{S_i} \in \cf_\cz$ and $u^\top \bm{\delta}_f = \langle f, \sum_i u_i \hat d_{S_i}\rangle_{\cf_\cz}$. By Cauchy--Schwarz in $\cf_\cz$,
\[
    (u^\top \bm{\delta}_f)^2 \;\leq\; \|f\|^2_{\cf_\cz}\,\Big\|\sum_i u_i \hat d_{S_i}\Big\|^2_{\cf_\cz} \;=\; \|f\|^2_{\cf_\cz}\,u^\top \widetilde K\, u,
\]
where the last equality uses $\langle \hat d_{S_i}, \hat d_{S_j}\rangle_{\cf_\cz} = \widetilde K_{ij}$. Applying this with $u = \hat e_j$ for $j = 1, \ldots, m$ and summing,
\[
    \sum_{j=1}^m (\hat e_j^\top \bm{\delta}_f)^2 \;\leq\; \|f\|^2_{\cf_\cz}\, \sum_{j=1}^m \hat e_j^\top \widetilde K \hat e_j \;=\; \|f\|^2_{\cf_\cz}\,\mathrm{tr}\!\big(P_m\widetilde K P_m\big).
\]
The left-hand side equals $\|P_m\bm{\delta}_f\|_2^2 = \sum_{i=1}^n (P_m\bm{\delta}_f)_i^2$ since $\{\hat e_j\}_{j=1}^m$ is an orthonormal basis for the range of $P_m$. Dividing by $n$,
\begin{equation}
\label{eq:proj_bound}
    \frac{1}{n}\sum_{i=1}^n (P_m\bm{\delta}_f)_i^2 \;\leq\; \frac{\|f\|^2_{\cf_\cz}}{n}\,\mathrm{tr}\!\big(P_m\widetilde K P_m\big).
\end{equation}

\textbf{Step 4: Combine.}
Combining~\eqref{eq:hsic_lower} (which gives $\mathrm{tr}(P_m\widetilde K P_m) \leq n\,\widehat{\hsic}/\hat\lambda_m$) with~\eqref{eq:proj_bound},
\[
    \frac{1}{n}\sum_{i=1}^n (P_m\bm{\delta}_f)_i^2 \;\leq\; \frac{\|f\|^2_{\cf_\cz}}{\hat\lambda_m}\,\widehat{\hsic}(Z, S),
\]
which is the second inequality of~\eqref{eq:hsic_gdp_bound} since $\widehat m_f = P_m\bm{\delta}_f$. The first inequality, $\widehat\Delta_{\mathrm{GDP}}(f)^2 = \big(n^{-1}\sum_i |(\widehat m_f)_i|\big)^2 \leq n^{-1}\sum_i (\widehat m_f)_i^2$, is the Cauchy--Schwarz inequality (equivalently, Jensen applied to $x \mapsto x^2$).
\end{proof}

The population counterpart of this finite-sample spectral control is the upper bound of Theorem~\ref{cor:hsic-mmd-equivalence}, which handles the compact kernel operator on $\cs$ through its spectral tail $\rho_m$ without assuming a spectral gap.

%======================================================================
\section{Additional Notation and Technical Lemmas}
\label{app:notation}
%======================================================================

\begin{lemma}[Fubini's theorem for RKHS-valued integrals]
\label{lem:fubini}
{\sloppy
Let $k$ be a bounded kernel, $\sup_{x \in \mathcal{X}} k(x, x) \leq \kappa < \infty$, with RKHS $\mathcal{F}$. For any random variable $X$ on $\mathcal{X}$, the mean embedding $\mu_X = \mathbb{E}[k(X, \cdot)]$ exists as a Bochner integral in $\mathcal{F}$. Moreover, for any $\mathcal{F}$-valued Bochner-integrable random element $U$ and any $v \in \mathcal{F}$, $\langle \mathbb{E}[U], v\rangle_{\mathcal{F}} = \mathbb{E}[\langle U, v\rangle_{\mathcal{F}}]$; in particular $\langle \mu_X, v\rangle_{\mathcal{F}} = \mathbb{E}[v(X)]$.\par}
\end{lemma}

\begin{proof}
By the reproducing property,
\[
    \|k(x, \cdot)\|_\mathcal{F} = \sqrt{k(x,x)} \leq \sqrt{\kappa},
    \quad\text{so}\quad
    \mathbb{E}[\|k(X, \cdot)\|_\mathcal{F}] \leq \sqrt{\kappa} < \infty,
\]
guaranteeing existence of the Bochner integral.
The exchange of expectation and inner product follows from continuity of the inner product together with Fubini's theorem.
\end{proof}

\subsection{Discrete sensitive attribute: coincidence of the two formulations}
\label{app:discrete-coincidence}

When $S$ takes finitely many values, the joint and conditional-integral formulations of representation-level fairness reduce to the same aggregated two-sample MMD object, which motivates studying the continuous case in the main text.

\begin{lemma}[Discrete sensitive attribute coincidence]\label{lem:discrete-coincidence}
Use the notation: $\cf_\cz$ is the RKHS on $\cz$ with characteristic kernel $k_\cz$, $\mu_Z:=\mathbb E[k_\cz(Z,\cdot)]\in\cf_\cz$ is the marginal mean embedding, and $\mu_{Z\mid s_k}:=\mathbb E[k_\cz(Z,\cdot)\mid S=s_k]\in\cf_\cz$ is the conditional mean embedding. Let $\mathcal S=\{s_1,\dots,s_K\}$ be finite with $\pi_k:=\Pr(S=s_k)>0$, let $k_\cs$ be the strictly positive-definite sensitive-attribute kernel, write $\delta_k:=\mu_{Z\mid s_k}-\mu_Z$, and let $\mathcal I_{\mathrm{MMD}^2}(Z;S):=\mathbb E_S[\mathrm{MMD}^2(P_{Z\mid S},P_Z)]$ be the conditional-integral functional with local discrepancy $d=\mathrm{MMD}^2$. Then
\[
\mathrm{HSIC}(Z,S)=\sum_{k,k'=1}^K \pi_k\pi_{k'}\,k_\cs(s_k,s_{k'})\,\langle \delta_k,\delta_{k'}\rangle_{\cf_\cz}.
\]
$\mathrm{HSIC}(Z,S)$ and $\mathcal I_{\mathrm{MMD}^2}(Z;S)$ both vanish if and only if $P_{Z\mid S=s_k}=P_Z$ for every $k$; both reduce to aggregated two-sample MMDs.
\end{lemma}

\begin{proof}
Expand $\mathrm{HSIC}$ via $\mu_{P_{Z,S}}-\mu_{P_Z\otimes P_S}=\mathbb E_S[(\mu_{Z\mid S}-\mu_Z)\otimes k_{\mathcal S}(S,\cdot)]$ and integrate the finite measure on $\mathcal S$.
\end{proof}

\subsection{Bandwidth scaling of the spectral constant}
\label{app:lambda2_scaling}

\begin{remark}[Bandwidth scaling of the full-resolution constant $\hat\lambda_S$]
\label{rem:lambda2_scaling}
{\sloppy
The full-resolution constant $\hat\lambda_S = \lambda_{\min}^+(n^{-1}\widetilde L)$ of Theorem~\ref{thm:hsic_gdp} at $m=r$ reflects the conditioning of the centered Gram matrix on $S$. As the sensitive-kernel bandwidth tends to zero and $L \to I_n$, $\widetilde L \to H$ (whose positive eigenvalues are all $1$), and $\hat\lambda_S \to 1/n$. As the bandwidth tends to infinity and $L \to \mathbf{1}\mathbf{1}^\top$, $\widetilde L \to 0$ and $\hat\lambda_S \to 0$. Very large bandwidths therefore make the conditional-gap bound loose, while extremely small bandwidths give a finite-sample point-mass notion of conditioning; the median heuristic selects an intermediate regime. Standard results on Gaussian-kernel Gram matrices for samples from densities with compact support \citep[e.g.,][]{koltchinskii2000random} confirm a polynomial dependence on $\sigma$ between these extremes; in our experiments at $n = 1500$ this yields $\hat\lambda_S \sim 10^{-9}$--$10^{-7}$ across the bandwidth sweep (Appendix~\ref{sec:bound_tightness}).\par}
\end{remark}

%======================================================================
\section{Regularized empirical kHGR}
\label{app:khgr_supplement}
%======================================================================

This appendix records a regularized empirical bound linking $\widehat{\hsic}$ to a Tikhonov-regularized empirical kHGR. The unregularized empirical kHGR requires inverting empirical covariance operators with arbitrarily small positive eigenvalues, so a meaningful finite-sample statement is naturally regularized; see \citet{fukumizu2007statistical} for the kCCA interpretation. The FRHSIC guarantees in the main text do not depend on this result.

\paragraph{Empirical covariance operators}
Define the empirical centered cross-covariance operator $\hat\Sigma_{ZS}: \cf_\cs \to \cf_\cz$ by
\[
    \hat\Sigma_{ZS} = \frac{1}{n}\sum_{i=1}^n \big(k_\cz(Z_i, \cdot) - \hat\mu_Z\big) \otimes \big(k_\cs(S_i, \cdot) - \hat\mu_S\big),
\]
where $\hat\mu_Z, \hat\mu_S$ are the empirical mean embeddings, and define $\hat\Sigma_{ZZ}, \hat\Sigma_{SS}$ analogously. All three operators are finite-rank (rank at most $n$), hence Hilbert--Schmidt without further assumptions on the kernels. By a direct expansion, $\widehat{\hsic}(Z, S) = \|\hat\Sigma_{ZS}\|_{\mathrm{HS}}^2$.

\paragraph{Regularized empirical kHGR}
{\sloppy For ridge parameters $\eta_Z, \eta_S > 0$, define the Tikhonov-regularized empirical kHGR\par}
\[
    \widehat{\mathrm{kHGR}}_{\eta_Z, \eta_S}(Z, S) \;:=\; \big\|(\hat\Sigma_{ZZ} + \eta_Z I)^{-1/2}\,\hat\Sigma_{ZS}\,(\hat\Sigma_{SS} + \eta_S I)^{-1/2}\big\|_{\mathrm{op}}.
\]

\begin{proposition}[Regularized empirical sandwich bound]
\label{prop:khgr_reg_sandwich}
Let $k_\cz, k_\cs$ be characteristic kernels and let $\eta_Z, \eta_S > 0$. Then
\begin{equation}
\label{eq:khgr_reg_bound}
    \widehat{\mathrm{kHGR}}_{\eta_Z, \eta_S}(Z, S)^2 \;\leq\; \frac{\widehat{\hsic}(Z, S)}{\eta_Z\,\eta_S}.
\end{equation}
\end{proposition}

\begin{proof}
Set $A = (\hat\Sigma_{ZZ} + \eta_Z I)^{-1/2}$, $B = \hat\Sigma_{ZS}$, $C = (\hat\Sigma_{SS} + \eta_S I)^{-1/2}$. Each operator $\hat\Sigma_{ZZ}, \hat\Sigma_{SS}$ is positive semi-definite, so $\hat\Sigma_{ZZ} + \eta_Z I \succeq \eta_Z I$ and $\hat\Sigma_{SS} + \eta_S I \succeq \eta_S I$. Consequently
\[
    \|A\|_{\mathrm{op}} \;\leq\; 1/\sqrt{\eta_Z}, \qquad \|C\|_{\mathrm{op}} \;\leq\; 1/\sqrt{\eta_S}.
\]
By the operator-norm $\leq$ Hilbert--Schmidt-norm inequality and the standard submultiplicative bound $\|A B C\|_{\mathrm{op}} \leq \|A\|_{\mathrm{op}}\,\|B\|_{\mathrm{HS}}\,\|C\|_{\mathrm{op}}$ (valid because the Hilbert--Schmidt norm dominates the operator norm),
\[
    \widehat{\mathrm{kHGR}}_{\eta_Z, \eta_S}(Z, S)
    = \|A B C\|_{\mathrm{op}}
    \;\leq\; \|A\|_{\mathrm{op}}\,\|B\|_{\mathrm{HS}}\,\|C\|_{\mathrm{op}}
    \;\leq\; \frac{\|\hat\Sigma_{ZS}\|_{\mathrm{HS}}}{\sqrt{\eta_Z\,\eta_S}}.
\]
Squaring and using $\widehat{\hsic}(Z, S) = \|\hat\Sigma_{ZS}\|_{\mathrm{HS}}^2$ gives~\eqref{eq:khgr_reg_bound}.
\end{proof}

The proposition makes no claim about the unregularized limit $\eta_Z, \eta_S \downarrow 0$, for which the spectrum of the empirical covariance operators must be controlled separately. Population-level zero-equivalence of HSIC and kHGR under characteristic kernels is independent of this regularization choice.

%======================================================================
\section{Experimental Details}
\label{app:experiments}
%======================================================================

\subsection{Datasets}

\paragraph{Adult Income}
The UCI Adult dataset contains $n \approx 30{,}000$ records (after removing missing values).
Features include numeric variables and one-hot encoded categoricals ($d = 107$ after encoding).
The target is whether income exceeds \$50K/year.
The sensitive attribute is age (continuous, range 17--90).
All features and the sensitive attribute are min-max scaled to $[0, 1]$.

\paragraph{ACS Income}
The ACS Income dataset \citep{ding2021retiring} is drawn from the 2018 American Community Survey (California) via the folktables package.
We subsample $n = 20{,}000$ individuals.
The target is whether income exceeds \$50K/year.
The sensitive attribute is age (continuous, range 17--94).
Features include 9 demographic and employment variables.

\paragraph{MEPS}
The Medical Expenditure Panel Survey (Panel 19, 2015) \citep{romano2020achieving} contains $n \approx 13{,}000$ respondents after removing records with missing values (negative sentinel codes).
The target is healthcare utilization $\geq 10$ visits (binary classification).
The sensitive attribute is age (continuous, range 17--85).
Features include 19 variables covering demographics, health conditions, insurance status, and income.

\paragraph{Communities \& Crime}
The UCI Communities \& Crime dataset contains $n \approx 2{,}000$ communities.
The target is violent crimes per population (regression).
The sensitive attribute is racial composition (continuous).
Columns with ${>}20\%$ missing values are dropped; remaining missing values are imputed with column medians.

\paragraph{COMPAS}
The ProPublica COMPAS dataset contains $n \approx 6{,}000$ defendants after standard filtering \citep{angwin2016machine}.
The target is two-year recidivism (binary classification).
The sensitive attribute is age (continuous).
Features include prior counts, charge degree (one-hot), sex, and race (one-hot), yielding $d = 16$ features.

\subsection{Model Architecture}

Following \citet{kong2025fair}, the encoder $h$ is a 2-layer MLP: $d_X \to 50 \to 50$ with SELU activations (hidden dimension $H = 50$, output representation dimension $Z = 50$).
The prediction head $f$ is a linear layer: $50 \to 1$ (no hidden layers).
This architecture matches FREM exactly, ensuring a fair comparison.
The adversary in LAFTR and ADV uses the same 2-layer structure ($50 \to 50 \to n_{\mathrm{out}}$) with SELU activation.

\subsection{Training Details}

All models are trained for $200$ epochs with the Adam optimizer (learning rate $10^{-3}$, default $\beta_1 = 0.9$, $\beta_2 = 0.999$, $\varepsilon = 10^{-8}$, weight decay $0$), batch size $256$, and the standard PyTorch default initialization.
For FRHSIC, kernel bandwidths are set via the median heuristic: $\sigma_\mathcal{S}$ is computed once on the training data, and $\sigma_\mathcal{Z}$ is recomputed every 20 epochs on the current mini-batch of encoded representations.
For baselines (FREM, Reg-GDP, etc.), we use fixed $\sigma_Z = 1.0$ for the $Z$-kernel after min-max scaling to $[0, 1]$, following \citet{kong2025fair}; the additional $S$-smoothing bandwidths used by FREM ($\gamma = 0.5$) and Reg-GDP (Nadaraya--Watson bandwidth $0.2$) are taken from their source publications and described per-method below.
We do not modify per-method bandwidth conventions, ensuring each method is evaluated as published.

\paragraph{Train/validation/test splits}
For all real datasets we use an $80\%/20\%$ random train/test split. Hyperparameter selection (i.e., choosing $\lambda$ along the Pareto frontier) is performed using the same training split as the model fit; for each method and each dataset we sweep $\lambda$ over the full grid and report results across the resulting Pareto frontier (no separate validation split is used to pick a single $\lambda$, since the goal of the comparison is to characterize the fairness--accuracy tradeoff curve, not to choose one operating point). We repeat 5 times with different random splits and report mean $\pm$ standard deviation on the held-out test set.

\paragraph{Regularization-strength ($\lambda$) grids}
For all methods on all real datasets, we sweep
\[
    \lambda \in \{0.1,\; 1.0,\; 10.0,\; 100.0,\; 500.0\}.
\]
For the Equal Opportunity and multi-sensitive-attribute experiments (Appendix~\ref{app:additional_results}) the grid is $\lambda \in \{0.1, 1.0, 10.0, 50.0\}$.

\subsection{Baseline Implementations and Hyperparameter Grids}

\paragraph{Reg-GDP}
We implement the GDP regularizer from \citet{jiang2022generalized} using kernel-smoothed conditional expectations on a grid of 30 evenly-spaced points over the range of $S$. Reg-GDP uses two distinct bandwidths: a Gaussian $Z$-kernel bandwidth $\sigma_Z = 1.0$ (matching the FREM convention of \citet{kong2025fair} on min-max-scaled features) and a Nadaraya--Watson smoothing bandwidth $0.2$ on $S$ for the conditional expectation; both follow the values reported by \citet{jiang2022generalized}.

\paragraph{FREM}
We implement the weighted EIPM estimator from \citet{kong2025fair} using the MMD discriminator.
Kernel-smoothed weights are computed with a Gaussian kernel on $S$ with bandwidth $\gamma = 0.5$ (following \citet{kong2025fair}).
For computational tractability we subsample $n_{\mathrm{anchor}} = 32$ anchor points per batch for the EIPM computation; see Appendix~\ref{app:frem_anchors} below for details.

\paragraph{LAFTR}
We bin the continuous sensitive attribute into 4 quartiles and apply the adversarial FRL method of \citet{madras2018learning}.
The adversary is trained with alternating gradient steps (one adversary step per encoder step).
The adversary is a 2-layer MLP ($50 \to 50 \to 4$) with SELU activations.

\paragraph{ADV}
The continuous-target adversary of \citet{grari2022fairness}, with the same 2-layer architecture as LAFTR but with a single scalar output (predicting $S$); trained with one adversary step per encoder step.

\paragraph{MMD-binned}
Bins continuous $S$ into $n_{\mathrm{bins}} = 10$ equal-width bins and applies pairwise MMD across bins; bin centers serve as the smoothing grid.

\paragraph{dCor}
We use the unbiased empirical distance correlation of \citet{szekely2007measuring} as the fairness penalty; no further hyperparameters beyond $\lambda$.

\subsection{FREM 32-anchor subsampling}
\label{app:frem_anchors}

{\sloppy The FREM EIPM objective of \citet{kong2025fair} computes, at each mini-batch, an integral over $s$ of a weighted MMD term:\par}
\[
    \widehat{\mathrm{EIPM}}_n(Z; S) = \int_\cs \widehat{\mathrm{MMD}}^2\big(\widehat{P}_{Z\mid S=s},\, \widehat{P}_Z\big)\, d\widehat{P}_S(s).
\]
Estimating this expectation by Monte Carlo with all $n$ training points as anchors would require $O(n^3)$ work per batch (an $n \times n$ kernel matrix evaluated at each of $n$ anchors). For computational tractability we instead draw $n_{\mathrm{anchor}} = 32$ uniformly-random anchor points per mini-batch and average the MMD terms over those anchors. With batch size $256$ this is a $32/256 = 12.5\%$ Monte Carlo subsample and reduces the per-batch cost to $O(n_{\mathrm{anchor}} \cdot n^2)$. The number $32$ was chosen to match the FREM implementation in our codebase (see \texttt{experiments/real\_data.py}, function \texttt{train\_frem}); we report results with $n_{\mathrm{anchor}} = 32$ throughout.

\paragraph{FREM anchor-subsampling sensitivity (limitation)}
The 32-anchor subsampling we use for FREM is a computational approximation; an anchor-count sweep is the natural next step to characterize the sensitivity of FREM's reported Pareto frontier to the anchor budget, and we plan to include such a sweep in a follow-up version of this work. The per-epoch wall-clock cost grows linearly with the anchor count, so the principal cost of the sweep is compute rather than implementation.

\paragraph{Baseline bandwidth sensitivity (limitation)}
The bandwidths reported in Appendix~\ref{app:experiments} (FRHSIC median heuristic, FREM $\gamma = 0.5$ and $\sigma_Z = 1.0$, Reg-GDP Nadaraya--Watson bandwidth $0.2$ and $\sigma_Z = 1.0$) are the values reported in the source publications. We retain published bandwidths in this paper to avoid favoring any particular method by retuning, but a full bandwidth-sensitivity sweep for the FREM and Reg-GDP baselines is the natural next step, and we plan to include such a sweep in a follow-up version of this work.

\paragraph{Paired significance test (limitation)}
We report mean and standard deviation across $n=5$ random splits in Table~\ref{tab:real_data} but do not run a formal paired test (e.g., Wilcoxon signed-rank), since $5$ paired observations is too few for reliable inference at conventional significance levels. Increasing the number of repeats to enable a credible paired comparison is the natural next step and we plan to include such an analysis in a follow-up version of this work.

\subsection{Reproducibility checklist}
\label{app:repro}

We summarize the elements needed to reproduce the experiments end to end.

\begin{itemize}
\item \textbf{Random seed.} A single base seed \verb|SEED = 42| is set at the top of each experiment script for \verb|numpy|, \verb|torch|, and the train/test split RNG. The 5 repeats use seeds $\verb|SEED|, \verb|SEED|+1, \ldots, \verb|SEED|+4$. See \verb|experiments/real_data.py:27| and \verb|experiments/utils.py|.
\item \textbf{Train/validation/test split.} $80\%/20\%$ random train/test split; no separate validation split (rationale above). Repeated 5 times with different splits.
\item \textbf{Preprocessing.} Min--max scaling for the features and the sensitive attribute is fit on the \emph{training split only} and then applied to the held-out test split (\verb|experiments/real_data.py|, \verb|run_single_dataset|). All cells of Table~\ref{tab:real_data} in the main text are reported under this leakage-free scaling protocol; the resulting numbers agree with our earlier leakage-affected numbers to within standard deviations on essentially every cell, with the most visible shifts in the MMD and ADV rows of Crime (where the small sample size $n \approx 2000$ amplifies any scaling difference). Categoricals are one-hot encoded; rows with missing values are dropped except in Communities \& Crime, where columns with ${>}20\%$ missing are dropped and remaining missing values are imputed with column medians.
\item \textbf{Model architecture.} Encoder $h$: 2-layer MLP $d_X \to 50 \to 50$ with SELU activations. Predictor $f$: linear $50 \to 1$. Adversaries (LAFTR / ADV / MMD-binned): 2-layer MLP $50 \to 50 \to n_{\mathrm{out}}$ with SELU.
\item \textbf{Optimizer.} Adam, learning rate $10^{-3}$, $\beta_1=0.9, \beta_2=0.999$, $\varepsilon=10^{-8}$, weight decay $0$. PyTorch defaults.
\item \textbf{Batch size and epochs.} Batch size $256$, $200$ epochs.
\item \textbf{Kernel bandwidths.} FRHSIC: median heuristic for $\sigma_\cs$ (fixed at start of training, on training data) and for $\sigma_\cz$ (recomputed every 20 epochs on the current mini-batch). Baselines: fixed $\sigma_Z = 1.0$ on min-max-scaled $Z$, FREM $\gamma = 0.5$, Reg-GDP NW bandwidth $0.2$.
\item \textbf{$\lambda$ grids.} Main experiments: $\lambda \in \{0.1, 1.0, 10.0, 100.0, 500.0\}$. EO and multi-sensitive: $\lambda \in \{0.1, 1.0, 10.0, 50.0\}$.
\item \textbf{Baseline hyperparameter grids.} See Appendix~\ref{app:experiments}, ``Baseline Implementations''; all baselines use the same encoder, predictor, optimizer, batch size, and number of epochs as FRHSIC, varying only the per-method fairness-penalty hyperparameters listed above. The adversary step ratio (LAFTR, ADV) is fixed at $1\!:\!1$.
\item \textbf{FREM anchor count.} $n_{\mathrm{anchor}} = 32$ uniformly-sampled anchors per mini-batch; see Appendix~\ref{app:frem_anchors}.
\item \textbf{Hardware.} All experiments were run on a single workstation with one NVIDIA RTX-class GPU; the experiment code falls back to CPU if no CUDA device is available (see the device selector in \verb|experiments/utils.py|). Runtime numbers in Section~\ref{sec:experiments} are wall-clock per-epoch on this single device, averaged over 5 epochs after a warm-up epoch. Training the full $\lambda$-sweep across all methods and datasets fits within $\sim$24 GPU-hours.
\item \textbf{Software.} Python 3.11, PyTorch 2.x, NumPy 1.26, scikit-learn 1.3, folktables (for ACS Income). See \verb|experiments/README.md| for the exact versions used at submission time.
\item \textbf{Code availability.} All scripts to reproduce the figures and tables are in the \verb|experiments/| directory of the public code repository (\url{https://github.com/Yijin911/FRHSIC}); each main-paper figure/table is generated by a single named script (e.g.\ \verb|run_single_method.py|, \verb|regen_pareto_mi.py|, \verb|runtime.py|, \verb|convergence.py|).
\end{itemize}

\subsection{Pareto summary: lowest GDP within 1\% of the unfair-baseline performance}
\label{app:pareto_summary}

For a single-number summary of the fairness--accuracy tradeoff that complements the matched-performance comparison in Section~\ref{sec:experiments}, Table~\ref{tab:pareto_summary} reports, for each (method, dataset) pair, the operating point on the $\lambda$-sweep that attains the \emph{lowest GDP} subject to the constraint that the prediction performance is within $1\%$ of the unfair baseline (Acc $\geq 0.99 \cdot \mathrm{Acc}_{\mathrm{Unfair}}$ for classification; MSE $\leq 1.01 \cdot \mathrm{MSE}_{\mathrm{Unfair}}$ for regression). Entries are read off the Pareto sweep underlying Figure~\ref{fig:pareto}. ``---'' denotes that no $\lambda$ in our grid satisfies the constraint for that method on that dataset. Values are means across $5$ random splits.

\begin{table}[!htbp]
\centering
\caption{Pareto summary: $(\text{Acc / MSE}, \text{GDP})$ at the lowest-GDP operating point with prediction performance within $1\%$ of the unfair baseline. Adult/ACS/MEPS/COMPAS report Acc; Crime reports $\mathrm{MSE} \times 10^2$. ``---'' indicates no $\lambda$ in the sweep satisfies the constraint. Read alongside Figure~\ref{fig:pareto}, since a tight $1\%$ band can exclude operating points that are visually close on the frontier.}
\label{tab:pareto_summary}
\renewcommand{\arraystretch}{0.95}
\small
\setlength{\tabcolsep}{4pt}
\begin{tabular}{lccccc}
\toprule
Method & Adult & ACS Income & MEPS & Crime & COMPAS \\
\midrule
Unfair (ref.)   & ($.849$, $.749$) & ($.790$, $.514$) & ($.804$, $.626$) & ($1.91$, $.029$) & ($.654$, $.123$) \\
\midrule
FRHSIC (Ours)   & ($.845$, $.232$) & ($.792$, $.399$) & ($.803$, $.517$) & ($1.91$, $.029$) & ($.655$, $.107$) \\
FREM            & ($.845$, $.167$) & ($.787$, $.274$) & ($.800$, $.317$) & ($1.87$, $.023$) & ($.648$, $.101$) \\
Reg-GDP         & ($.842$, $.082$) & ($.787$, $.292$) & ($.801$, $.370$) & ($1.91$, $.025$) & ($.650$, $.079$) \\
ADV             & ($.841$, $.466$) & ($.785$, $.458$) & ($.803$, $.480$) & ($1.92$, $.029$) & ($.649$, $.095$) \\
MMD (binned)    & ($.847$, $.126$) & ($.783$, $.157$) & ($.800$, $.407$) & ---              & ($.651$, $.082$) \\
LAFTR (binned)  & ($.842$, $.385$) & ($.787$, $.474$) & ($.803$, $.530$) & ($1.89$, $.029$) & ($.655$, $.122$) \\
dCor            & ($.846$, $.194$) & ($.783$, $.203$) & ($.800$, $.487$) & ($1.92$, $.023$) & ($.653$, $.091$) \\
\bottomrule
\end{tabular}
\end{table}

\paragraph{Caveats}
The 1\% accuracy band is intentionally tight, and on some datasets (notably COMPAS, where multiple methods cluster near the unfair-baseline accuracy) the band excludes operating points that are visually close on the Pareto frontier (Figure~\ref{fig:pareto}). Reg-GDP can drive GDP to near zero on Adult and ACS, but only at accuracy losses well outside the 1\% band; those points are visible on Figure~\ref{fig:pareto} but do not appear in the table above. The single-number summary should therefore be read alongside the full Pareto curves rather than as a standalone ranking.

%======================================================================
\section{Additional Experimental Results}
\label{app:additional_results}
%======================================================================

\subsection{Synthetic validation of the conditional-gap bound}
\label{app:synthetic_validation}

On a controlled synthetic classification setup, enforcing HSIC drives GDP monotonically to zero at stable accuracy, confirming the conditional-gap bound of Theorem~\ref{thm:hsic_gdp}.
We generate $n = 5000$ samples with $S \sim \mathrm{Uniform}(0, 1)$, $X = (S + \epsilon_1, \epsilon_2)$ where $\epsilon_1 \sim \mathcal{N}(0, 0.09)$, $\epsilon_2 \sim \mathcal{N}(0, 1)$, and $Y \sim \mathrm{Bernoulli}(1/(1+e^{-X_1}))$.
The encoder $h$ is a 2-layer MLP (hidden dimension 64, representation dimension 8) and the prediction head $f$ is linear; a smaller architecture than the real-data experiments isolates the effect of the HSIC penalty in a controlled setting.
We train with objective $\mathcal{L}_{\mathrm{CE}}(f \circ h) + \lambda \cdot \widehat{\mathrm{HSIC}}(Z, S)$ for $\lambda \in \{0, 0.1, 0.5, 1, 5, 10, 50\}$.
As $\lambda$ increases, GDP decreases monotonically toward zero (Figure~\ref{fig:synthetic}a), GDP and HSIC decrease together in a monotone relationship consistent with the upper bound in Theorem~\ref{thm:hsic_gdp} (Figure~\ref{fig:synthetic}c), and accuracy remains stable (around $0.62$), indicating that the sensitive attribute contributes primarily to unfair discrimination rather than useful predictive signal in this setting.

\begin{figure}[t]
    \centering
    \includegraphics[width=\textwidth]{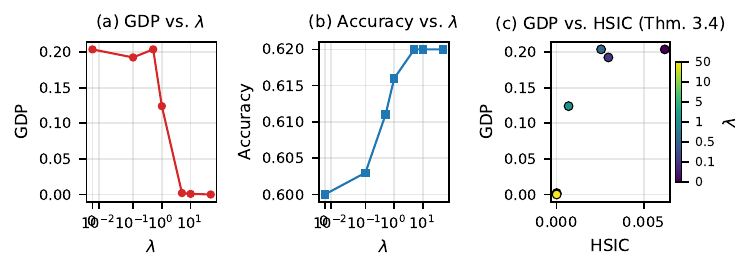}
    \caption{Synthetic validation. (a) GDP decreases as $\lambda$ increases. (b) Accuracy remains stable across regularization strengths. (c) GDP vs.\ HSIC shows a monotone relationship, consistent with the conditional-gap-control interpretation of Theorem~\ref{thm:hsic_gdp} (the plot displays monotonic dependence, not the LHS/RHS of the bound).}
    \label{fig:synthetic}
\end{figure}

\subsection{Equal Opportunity and multiple sensitive attributes}
\label{sec:eo}

The HSIC penalty extends to Equal Opportunity \citep{madras2018learning}, which requires fairness conditional on the positive outcome $Y=1$, by restricting the estimator to the $n_1$ samples with $Y_i=1$:
\[
    \widehat{\mathrm{HSIC}}_{\mathrm{EO}}(Z, S) = \frac{1}{n_1^2}\,\mathrm{tr}(K_1 H_1 L_1 H_1),
\]
where $K_1, L_1, H_1$ are the kernel and centering matrices restricted to that subset. It extends to multiple continuous sensitive attributes by taking a product kernel on $\mathcal S_1\times\mathcal S_2$ (Appendix~\ref{app:multi_s}).

We evaluate the EO extension on Adult and COMPAS, comparing FRHSIC-DP (HSIC on all samples) to FRHSIC-EO (HSIC on $Y=1$ subset only).
Table~\ref{tab:eo} reports accuracy and EO-GDP (GDP restricted to the $Y=1$ subset) across $\lambda$ values.
On COMPAS, the DP variant achieves comparable or better EO-GDP than the EO variant at similar accuracy across all $\lambda$ values.
On Adult, the EO variant achieves slightly lower EO-GDP at high $\lambda$ (0.352 vs.\ 0.418 at $\lambda = 50$), but the DP variant simultaneously reduces full GDP from 0.889 to 0.240---a broader fairness guarantee.
Overall, enforcing full demographic parity via HSIC provides a strong baseline for Equal Opportunity, with the dedicated EO variant offering marginal gains on larger datasets.

\begin{table}[ht]
\centering
\caption{Equal Opportunity results: FRHSIC-DP vs.\ FRHSIC-EO. EO-GDP is GDP restricted to $Y=1$ samples. FRHSIC-DP achieves comparable EO fairness while also enforcing DP.}
\label{tab:eo}
\begin{tabular}{llccc}
\toprule
Dataset & $\lambda$ & Method & Acc & EO-GDP ($\downarrow$) \\
\midrule
\multirow{8}{*}{Adult}
& \multirow{2}{*}{0.1} & DP & 0.851 & 0.453 \\
& & EO & 0.849 & 0.406 \\
& \multirow{2}{*}{1.0} & DP & 0.846 & 0.403 \\
& & EO & 0.850 & 0.416 \\
& \multirow{2}{*}{10.0} & DP & 0.846 & 0.437 \\
& & EO & 0.850 & 0.392 \\
& \multirow{2}{*}{50.0} & DP & 0.840 & 0.418 \\
& & EO & 0.849 & 0.352 \\
\midrule
\multirow{8}{*}{COMPAS}
& \multirow{2}{*}{0.1} & DP & 0.659 & 0.067 \\
& & EO & 0.662 & 0.070 \\
& \multirow{2}{*}{1.0} & DP & 0.662 & 0.070 \\
& & EO & 0.662 & 0.075 \\
& \multirow{2}{*}{10.0} & DP & 0.657 & 0.064 \\
& & EO & 0.655 & 0.065 \\
& \multirow{2}{*}{50.0} & DP & 0.599 & 0.033 \\
& & EO & 0.659 & 0.076 \\
\bottomrule
\end{tabular}
\end{table}

\subsection{Multiple Sensitive Attributes}
\label{app:multi_s}

We test FRHSIC with two continuous sensitive attributes (age and hours-per-week) on Adult.
Table~\ref{tab:multi_s} compares two approaches: (1) a product kernel on joint $(S_1, S_2)$ (FRHSIC-Joint), and (2) a sum of per-attribute HSIC terms (FRHSIC-Sum).
Both approaches reduce GDP with respect to each attribute as $\lambda$ increases while preserving accuracy, confirming that HSIC naturally extends to the multi-attribute setting without modification.
The Joint approach achieves slightly lower GDP at $\lambda = 50$ (Age: $0.306$ vs.\ $0.678$; Hours: $0.517$ vs.\ $0.759$), likely because the product kernel captures cross-attribute structure.
Figure~\ref{fig:multi_s} visualizes the Pareto frontiers.

\begin{table}[ht]
\centering
\caption{Multiple sensitive attributes on Adult (age + hours-per-week). GDP is reported separately for each attribute. Both product-kernel (Joint) and sum-of-HSIC (Sum) approaches reduce GDP for both attributes as $\lambda$ increases.}
\label{tab:multi_s}
\begin{tabular}{lccc}
\toprule
Method ($\lambda$) & Acc & GDP(Age) ($\downarrow$) & GDP(Hours) ($\downarrow$) \\
\midrule
Unfair & 0.843 & 0.837 & 1.142 \\
FRHSIC-Joint ($\lambda=1$) & 0.844 & 0.779 & 1.029 \\
FRHSIC-Joint ($\lambda=10$) & 0.846 & 0.553 & 0.857 \\
FRHSIC-Joint ($\lambda=50$) & 0.851 & 0.306 & 0.517 \\
FRHSIC-Sum ($\lambda=1$) & 0.847 & 0.753 & 1.055 \\
FRHSIC-Sum ($\lambda=10$) & 0.848 & 0.485 & 0.826 \\
FRHSIC-Sum ($\lambda=50$) & 0.849 & 0.678 & 0.759 \\
\bottomrule
\end{tabular}
\end{table}

\begin{figure}[ht]
    \centering
    \includegraphics[width=0.9\textwidth]{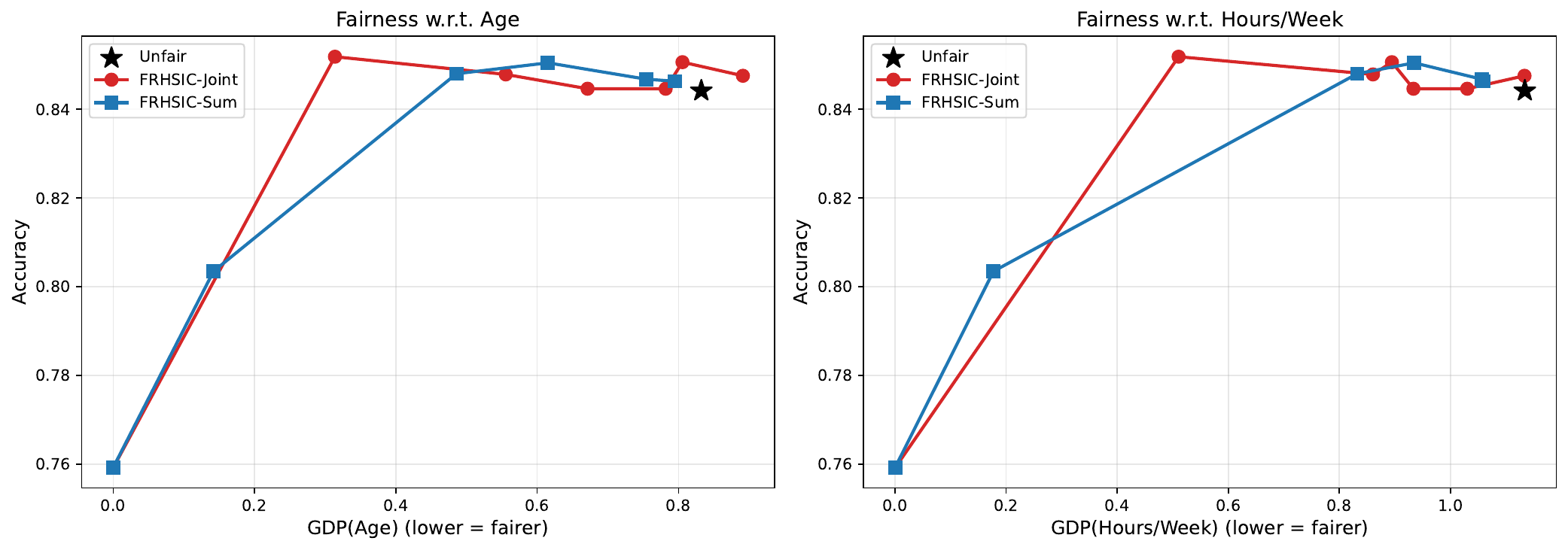}
    \caption{Multiple sensitive attributes: fairness--accuracy Pareto frontiers for FRHSIC-Joint (product kernel) and FRHSIC-Sum (sum of per-attribute HSIC), evaluated on Adult with age and hours-per-week as sensitive attributes.}
    \label{fig:multi_s}
\end{figure}

%======================================================================
\section{Structural Causal Models}
\label{app:scm}
%======================================================================

\subsection{Structural Causal Models and Counterfactual Fairness}
\label{sec:scm_prelim}

Following \citet{pearl2009causality} and \citet{kusner2017counterfactual}, a Structural Causal Model (SCM) over $(X, S, Y)$ specifies structural equations
\[
    S = U_S, \qquad X = f_X(S, U_X), \qquad Y = f_Y(X, S, U_Y),
\]
with jointly independent exogenous noise $U_X, U_S, U_Y$. The counterfactual prediction $\widehat Y_{S \leftarrow s'}$ is the value of $\widehat Y$ obtained by replacing $S = U_S$ with $S = s'$ in the structural equations and propagating downstream. A predictor $\widehat Y$ is \emph{counterfactually fair} \citep{kusner2017counterfactual} if, for every observable $(X = x, S = s)$ and every alternative $s'$,
\[
    \widehat Y_{S \leftarrow s'} = \widehat Y_{S \leftarrow s} \quad \text{a.s.\ given } X = x, S = s.
\]

%======================================================================
\section{Additional Theoretical Augmentations}
\label{app:augmentations}
%======================================================================

This section collects two augmentations referenced from the main text: a necessity-for-counterfactual-fairness result and a uniform-concentration bound that bridges training-sample HSIC to its population counterpart.

\subsection{Necessity for Counterfactual Fairness}
\label{sec:causal}

The joint-distribution route to independence also clarifies the relationship between HSIC-fairness and the causal-fairness literature.

\begin{proposition}[HSIC-fairness as a necessary condition for counterfactual fairness]
\label{prop:causal_necessary}
Assume the SCM in Appendix~\ref{sec:scm_prelim}. Suppose the encoder $h$ is a deterministic function so that $Z = h(X) = h(f_X(S, U_X))$. If the prediction $\widehat Y = f(Z)$ is counterfactually fair (Appendix~\ref{sec:scm_prelim}) for every $f$ in a class $\mathcal{F}$ rich enough to identify equality of distributions in $\cz$ (e.g., $\mathcal{F}$ is the unit ball of an RKHS with characteristic kernel $k_\cz$), then $\hsic(Z, S) = 0$.
\end{proposition}

\begin{proof}
Counterfactual fairness for every $f \in \mathcal{F}$ implies, in particular, that for all $s' \in \mathcal{S}$ in the support of $S$ and all $f \in \mathcal{F}$,
\[
    \mathbb{E}[f(Z_{S \leftarrow s'})] = \mathbb{E}[f(Z)],
\]
where $Z_{S \leftarrow s'} = h(f_X(s', U_X))$. Since $S = U_S$ and $U_X$ are independent under the SCM, and the marginal of $Z_{S \leftarrow s'}$ over $U_X$ has the same distribution as the marginal of $Z$ given $S = s'$ (because intervening with $S = s'$ matches the conditional given $S = s'$ when $S$ is exogenous \citep[Section~3.2.2]{pearl2009causality}), we obtain
\[
    \mathbb{E}[f(Z) \mid S = s'] = \mathbb{E}[f(Z)] \quad \text{for $P_S$-a.e.\ } s'.
\]
By Proposition~\ref{thm:subgroup_equivalence}, this is condition (4), equivalent to $Z\perp S$ and hence, for the characteristic product kernel, to $\hsic(Z, S) = 0$.
\end{proof}

In words, if a representation can support counterfactually fair predictions for every (sufficiently rich) prediction head, the representation must already satisfy the observational-independence condition that HSIC enforces. The converse does not hold: $\hsic(Z, S) = 0$ ensures that the marginal distribution of $f(Z)$ does not depend on $S$, but counterfactual fairness is a strictly pointwise statement about individual-level interventions, which can fail when the structural dependence on $S$ is observationally invisible (e.g., $Z = S \oplus U_X$ with binary $S, U_X$ uniform yields $Z \perp S$ marginally but $Z_{S \leftarrow s'} \neq Z_{S \leftarrow s}$ pointwise). HSIC-fairness should therefore be understood as the strongest \emph{observational} fairness consequence of counterfactual fairness, not as a substitute for it.

\subsection{Uniform Concentration over the Encoder Class}
\label{sec:uniform_concentration}
\label{app:uniform-hsic-application}

The empirical Theorem~\ref{thm:hsic_gdp} controls the empirical conditional-mean gap $\frac{1}{n}\sum_i (P_m\bm{\delta}_f)_i^2$ through $\widehat{\hsic}(h(X), S)$ for a single encoder $h$. In practice, $h$ is learned from the same data, so the bound must hold uniformly over a hypothesis class $\ch$ of encoders. We invoke Corollary~23 of \citet{ni2024uniform} (an HSIC specialization of their Theorem~12), which we adapt to the one-sided encoder class arising in FRHSIC. We do not reprove that result; this appendix only maps FRHSIC onto it.

\paragraph{Mapping FRHSIC to the prior uniform concentration result}
First, $\hsic(Z,S)$ is the squared MMD between $P_{Z,S}$ and $P_Z\otimes P_S$ under the fixed product kernel $k\big((z,s),(z',s')\big)=k_\cz(z,z')\,k_\cs(s,s')$, and $\widehat{\hsic}_n(h(X),S)=n^{-2}\,\mathrm{tr}(K_hHLH)$ is exactly the empirical kernel-based two-sample statistic to which Theorem~12 of \citet{ni2024uniform} applies, with $(Z',S')\sim P_Z\otimes P_S$. Second, the encoder class $\ch$ acts only on the $\cz$-coordinate, inducing the transformed paired samples $\{(h(X_i),S_i)\}$ and the encoder-indexed product-kernel class $\mathcal K_{\ch}$. Third, the boundedness and Lipschitz hypotheses needed by \citet{ni2024uniform} (Assumption~20) are the bounded-kernel and Lipschitz-feature-map conditions stated in Proposition~\ref{prop:uniform_hsic}, which $k_\cz$, $k_\cs$, and $h\in\ch$ satisfy under the standing assumptions. Fourth, bounded linear maps over a compact parameter set, and fixed-architecture MLPs with bounded weights, bounded inputs, and Lipschitz activations, have empirical Gaussian (equivalently, up to logarithmic factors, Rademacher) complexity of order $n^{-1/2}$, so the complexity term in the imported bound is controlled; growing-width or unregularized networks need not satisfy this. Data-dependent kernel selection, which would enlarge $\mathcal K_{\ch}$ to a composite-kernel class, is outside the scope of this statement.

\begin{proposition}[Uniform concentration of $\widehat{\hsic}$ over an encoder class]
\label{prop:uniform_hsic}
Let $\ch$ be a class of measurable encoders $h: \mathcal{X} \to \cz$.
Suppose the kernels $k_\cz, k_\cs$ satisfy Assumption~20 of \citet{ni2024uniform}: they are bounded with $\sup_z k_\cz(z, z) \leq \nu_\cz$ and $\sup_s k_\cs(s, s) \leq \nu_\cs$, and the feature maps $z \mapsto k_\cz(z, \cdot)$ and $s \mapsto k_\cs(s, \cdot)$ are Lipschitz with constants $\ell_\cz, \ell_\cs > 0$, respectively (the constants $\nu_\cz,\nu_\cs,\ell_\cz,\ell_\cs$ are exactly those of Theorem~\ref{thm:uniform-hsic-frhsic}).
Let
\[
    \hat{\mathcal{G}}_n(\ch) := \mathbb{E}_\xi\!\left[\sup_{h \in \ch} \frac{1}{n}\sum_{i=1}^n \langle \xi_i, h(X_i)\rangle \,\bigg|\, \{X_i\}_{i=1}^n\right]
\]
denote the empirical Gaussian complexity of $\ch$ \citep[Equation~3.5]{ni2024uniform}, where $\xi_1, \ldots, \xi_n$ are i.i.d.\ $\mathcal{N}(0, I_{\dim(\cz)})$ random vectors; write $\mathcal{G}_n(\ch) := \mathbb{E}[\hat{\mathcal{G}}_n(\ch)]$ for its expectation, as in the main text.
Then for any $\delta \in (0, 1)$, with probability at least $1 - \delta$,
\begin{equation}
\label{eq:uniform_hsic}
\begin{aligned}
    \sup_{h \in \ch}\,\big|\widehat{\hsic}(h(X), S) &- \hsic(h(X), S)\big|
    \;\leq\;
    8\sqrt{2}\,\nu_\cz \nu_\cs\,\sqrt{\tfrac{\log(2/\delta)}{n}}
    + \frac{4\,\nu_\cz \nu_\cs}{n} \\
    &+ 48\sqrt{\pi}\,\max\{\nu_\cz\, \ell_\cs, \nu_\cs\, \ell_\cz\}\,\mathcal{G}_n(\ch).
\end{aligned}
\end{equation}
\end{proposition}

\begin{proof}
Apply Corollary~23 of \citet{ni2024uniform} to the two-sample statistic $\gamma_k^2$ on the product RKHS with reproducing kernel $k\big((z, s), (z', s')\big) = k_\cz(z, z')\,k_\cs(s, s')$, identified with $\hsic$ via $\hsic(Z, S) = \gamma_k^2((Z, S), (Z', S'))$ for $(Z', S') \sim P_Z \otimes P_S$ \citep[Equation~3]{ni2024uniform}. The bivariate function class $\ch \times \{\mathrm{id}_\cs\}$, where $\mathrm{id}_\cs$ is the identity on $\cs$, satisfies $\mathbb{E}[\mathcal{G}((\ch \times \{\mathrm{id}_\cs\})(\mathbf{X}, \mathbf{S}))] \leq \mathbb{E}[\hat{\mathcal{G}}_n(\ch)]$ by Lemma~19 of \citet{ni2024uniform}, since the second-coordinate class is a singleton: under the no-absolute-value convention adopted by \citet[Section~3.1]{ni2024uniform}, $\mathcal{G}(\{\mathrm{id}_\cs\}(\mathbf{S})) = \mathbb{E}_\xi\!\left[\frac{1}{n}\sum_i \langle \xi_i, S_i\rangle\right] = 0$ by symmetry of the Gaussian distribution. Substituting into the HSIC bound of Corollary~23 yields~\eqref{eq:uniform_hsic}.
\end{proof}

\begin{corollary}[Train-population HSIC bridge]
\label{cor:train_pop_hsic}
Under the assumptions of Proposition~\ref{prop:uniform_hsic}, for any encoder $\hat h \in \ch$ obtained from $n$ i.i.d.\ training samples and any $\delta \in (0, 1)$, with probability at least $1 - \delta$,
\begin{equation}
\label{eq:train_pop_hsic}
    \hsic(\hat h(X), S)
    \;\leq\;
    \widehat{\hsic}(\hat h(X), S) + \varepsilon_n(\delta),
\end{equation}
where
\[
    \varepsilon_n(\delta) := 8\sqrt{2}\,\nu_\cz \nu_\cs \sqrt{\frac{\log(2/\delta)}{n}} + \frac{4\,\nu_\cz \nu_\cs}{n} + 48\sqrt{\pi}\,\max\{\nu_\cz \ell_\cs, \nu_\cs \ell_\cz\}\,\mathcal{G}_n(\ch)
\]
is the uniform-concentration penalty from Proposition~\ref{prop:uniform_hsic}.
\end{corollary}

\begin{proof}
By Proposition~\ref{prop:uniform_hsic}, with probability at least $1 - \delta$,
\[
    \sup_{h \in \ch}\,\big|\widehat{\hsic}(h(X), S) - \hsic(h(X), S)\big| \leq \varepsilon_n(\delta).
\]
Specializing to the data-dependent $\hat h \in \ch$, $\hsic(\hat h(X), S) - \widehat{\hsic}(\hat h(X), S) \leq \varepsilon_n(\delta)$, which rearranges to~\eqref{eq:train_pop_hsic}.
\end{proof}

\begin{remark}[Train-test fairness bridge]
\label{rem:train_test_bridge}
Corollary~\ref{cor:train_pop_hsic} delivers a train-test fairness guarantee when combined with Theorem~\ref{thm:hsic_gdp} applied on an independent test sample. Theorem~\ref{thm:hsic_gdp} controls the empirical conditional-mean gap $\frac{1}{n}\sum_i (P_m\bm{\delta}_f)_i^2$, which by Jensen's inequality dominates the squared empirical demographic-parity gap $\widehat\Delta_{\mathrm{GDP}}(f)^2$ for any prediction head $f \in \mathcal{F}_{\mathcal{Z}}$. Concretely, given a test sample $\{(\widetilde Z_i, \widetilde S_i)\}_{i=1}^n$ independent of the training sample, Theorem~\ref{thm:hsic_gdp} gives the deterministic bound
\[
    \frac{1}{n}\sum_{i=1}^n \big(P_m^{\mathrm{test}}\,\bm{\delta}_f^{\mathrm{test}}\big)_i^2
    \;\leq\;
    \frac{\|f\|^2_{\cf_\cz}}{\widetilde\lambda_m}\;\widehat{\hsic}^{\mathrm{test}}(\hat h(X), S),
\]
where $\widehat{\hsic}^{\mathrm{test}}$, $\widetilde\lambda_m$, and the projection $P_m^{\mathrm{test}}$ are computed on the test sample. Applying Proposition~\ref{prop:uniform_hsic} once on training and once pointwise on test (the latter being the standard single-encoder $O(n^{-1/2})$ deviation, which is dominated by $\varepsilon_n$),
\[
    \widehat{\hsic}^{\mathrm{test}}(\hat h(X), S)
    \;\leq\;
    \hsic(\hat h(X), S) + \varepsilon_n(\delta)
    \;\leq\;
    \widehat{\hsic}^{\mathrm{train}}(\hat h(X), S) + 2\,\varepsilon_n(\delta)
\]
with probability at least $1 - 2\delta$. Hence, with the same probability,
\[
    \frac{1}{n}\sum_{i=1}^n \big(P_m^{\mathrm{test}}\,\bm{\delta}_f^{\mathrm{test}}\big)_i^2
    \;\leq\;
    \frac{\|f\|^2_{\cf_\cz}}{\widetilde\lambda_m}\,\big[\widehat{\hsic}^{\mathrm{train}}(\hat h(X), S) + 2\,\varepsilon_n(\delta)\big].
\]
This is the generalization statement for the FRHSIC fairness surrogate: training-sample HSIC, plus the uniform-concentration penalty, controls the held-out projected empirical second-moment conditional gap up to the test-sample spectral factor $1/\widetilde\lambda_m$.
\end{remark}

%======================================================================
\section{Additional Experiments}
\label{app:additional_experiments}
%======================================================================

\subsection{Representation-level mutual information}
\label{app:mi_curves}

As a representation-level diagnostic complementing the transfer experiment, FRHSIC tends to attain lower $\mathrm{MI}(Z,S)$ than Reg-GDP at comparable prediction performance, consistent with the learned representation carrying less information about $S$ rather than only making a single trained head fair.
Figure~\ref{fig:mi} reports $\mathrm{MI}(Z,S)$ versus prediction performance across the $\lambda$ sweep.
We treat $\mathrm{MI}(Z,S)$ as a diagnostic, not a primary fairness audit: the KSG estimator can underestimate MI for high-dimensional learned representations whose dependence structure violates the smoothness assumptions underlying nearest-neighbor entropy estimation \citep{gretton2012kernel}, so we interpret it together with HSIC, GDP, and downstream performance.
The adversarial methods illustrate this: ADV and LAFTR can report $\mathrm{MI}\approx0$ while their GDP remains high, which is inconsistent with $\mathrm{MI}=0\Rightarrow Z\perp S\Rightarrow\mathrm{GDP}=0$ and indicates that adversarially trained representations, often concentrated on low-dimensional manifolds, are poorly suited to nearest-neighbor entropy estimation.
COMPAS is omitted because the KSG estimator is unreliable on its low-dimensional, discrete-valued features.

\begin{figure}[t]
    \centering
    \includegraphics[width=\textwidth]{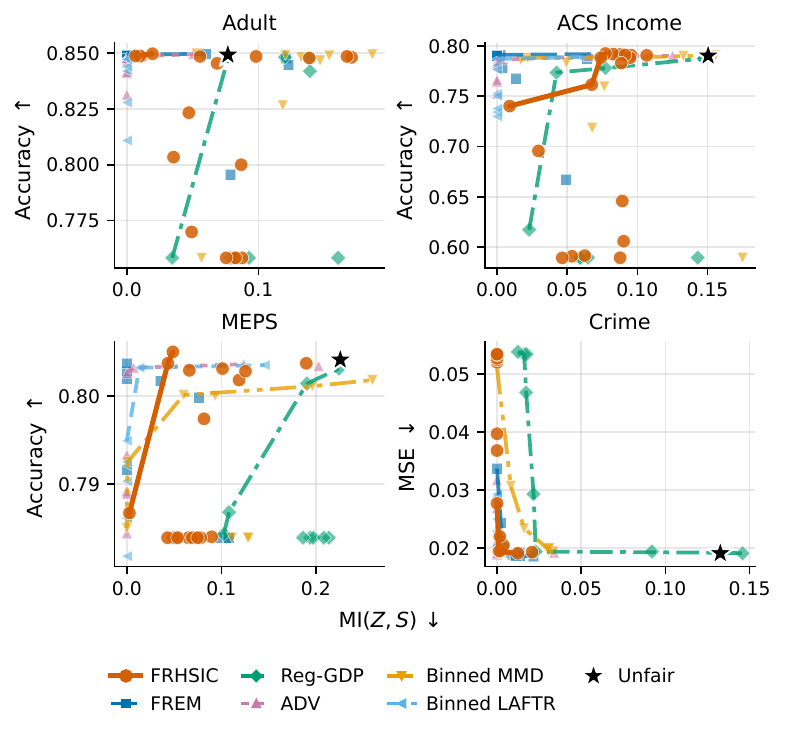}
    \caption{FRHSIC yields representations that leak less about $S$: it tends to attain lower $\mathrm{MI}(Z, S)$ than Reg-GDP at comparable prediction performance, consistent with representation-level rather than single-head fairness. Axes are $\mathrm{MI}(Z, S)$ (lower $=$ fairer) versus performance across $\lambda$.}
    \label{fig:mi}
\end{figure}

\subsection{Transfer Experiment}
\label{sec:transfer}

The key advantage of FRHSIC over Reg-GDP is \emph{transferability}: since HSIC enforces $Z \perp S$, any downstream prediction head built on the learned representation is guaranteed to be fair.
To test this empirically, we train representations using each method (at $\lambda = 10$), freeze the encoder, and fit four different downstream heads on the frozen representation: linear, 2-layer MLP, random forest, and SVM.
We then measure GDP for each head.

Figure~\ref{fig:transfer} shows the results on all five datasets; we highlight two representative findings.
On Crime ($\lambda = 100$), FRHSIC attains low GDP variance across heads (std $= 0.001$), comparable to or lower than Reg-GDP ($0.002$) and ADV ($0.005$).
On Adult ($\lambda = 10$), FRHSIC reduces cross-head GDP variance from $0.348$ (Unfair) to $0.244$, a $30\%$ reduction.
Notably, Reg-GDP achieves lower variance ($0.058$) on Adult but at the cost of collapsing accuracy to $0.758$ (Table~\ref{tab:real_data}).
Among methods that preserve accuracy, FRHSIC provides the most consistent fairness across downstream heads.
The MI results in Table~\ref{tab:real_data} provide complementary evidence: MI measures the total information about $S$ in $Z$, bounding the discrimination achievable by \emph{any} downstream head (including adversarial ones not tested here).
FRHSIC's low MI ($0.008$ on Adult, $0.015$ on Crime) thus guarantees fairness even for worst-case downstream use, while Reg-GDP's higher MI ($0.017$ on Adult, $0.056$ on Crime) leaves room for adversarial exploitation despite its GDP $= 0$ for the training head.

\begin{figure}[t]
    \centering
    \includegraphics[width=\textwidth]{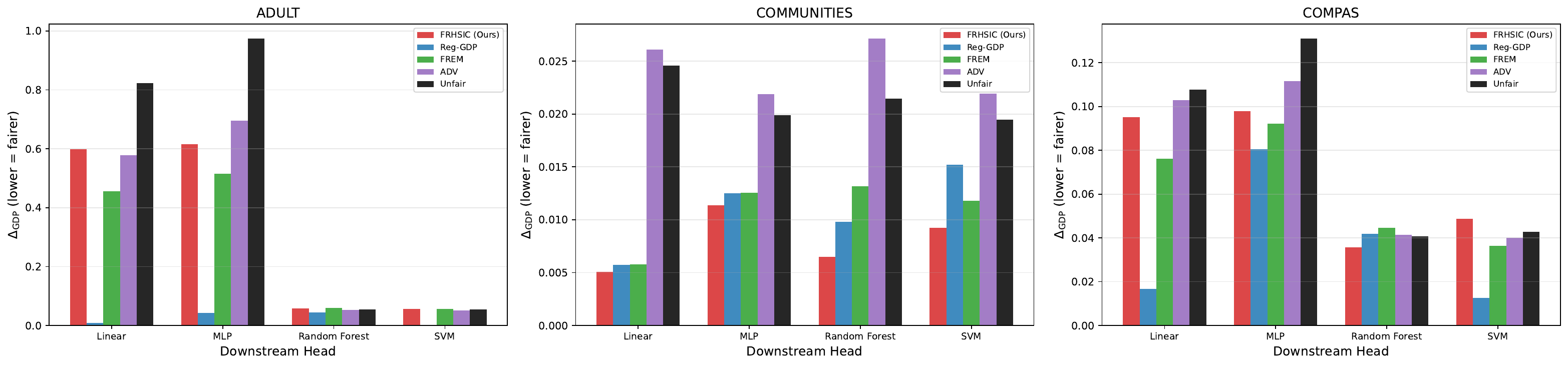}
    \caption{Transfer experiment: GDP of different downstream prediction heads trained on frozen representations. In these experiments FRHSIC maintains comparably low GDP across heads, while Reg-GDP and FREM exhibit more head-dependent variation.}
    \label{fig:transfer}
\end{figure}

\subsection{Empirical Tightness of the Theorem~\ref{thm:hsic_gdp} Bound}
\label{sec:bound_tightness}

We empirically evaluate the tightness of the bound in Theorem~\ref{thm:hsic_gdp} at full resolution ($m=r$, where $\hat\lambda_m=\hat\lambda_S$ is the smallest positive eigenvalue and $P_m\bm\delta_f=\bm\delta_f$) on two real datasets (Adult and Crime; we omit COMPAS because age in this dataset takes only a small number of distinct integer values, so the centered Gram matrix $\widetilde L$ on $S$ is near-low-rank and $\hat\lambda_S$ is numerically unstable). For each dataset and each regularization strength $\lambda$, we compute (i) the empirical RHS $\|f\|^2 / \hat\lambda_S \cdot \widehat{\hsic}(Z, S)$ on the test split, where $\hat\lambda_S$ is the smallest positive eigenvalue of $n^{-1}\widetilde L$, with $\widetilde L = H L H$ the centered Gram matrix on $S$ (eigenvalues below $10^{-6}\,\hat\lambda_{\max}$ are treated as numerical zero when identifying the smallest positive eigenvalue, mirroring standard practice for truncated/regularized kernel inverses); and (ii) the empirical LHS $\frac{1}{n}\sum_i \hat\delta_{f, i}^2$, averaged over RKHS test functions $f(\cdot) = k_\cz(z_0, \cdot)$ for randomly sampled anchors $z_0$. We plot both as a function of $\lambda$ on log--log axes.

Figure~\ref{fig:bound} shows the results. The bound is valid: the empirical RHS upper-bounds the LHS at every $\lambda$ on both datasets. The vertical gap between the two curves can span several orders of magnitude and reflects the well-known looseness of HSIC-based bounds when $\hat\lambda_S$ is small (the multiplier $1/\hat\lambda_S$ is large because $\hat\lambda_S$ scales as a small fraction of the largest eigenvalue of $n^{-1}\widetilde L$ for the Gaussian kernel at the median-heuristic bandwidth). Crucially, the empirical bound remains finite and non-vacuous --- in sharp contrast with the population-level version, where the analogous constant collapses to zero for Gaussian kernels on continuous $S$ and the bound becomes vacuous --- so the empirical reformulation of Theorem~\ref{thm:hsic_gdp} is quantitatively useful as a tractable surrogate. Choosing $m<r$ replaces $\hat\lambda_S$ by the larger $\hat\lambda_m$ and tightens the bound, at the cost of resolving only the top-$m$ sensitive directions.

\begin{figure}[ht]
    \centering
    \includegraphics[width=\textwidth]{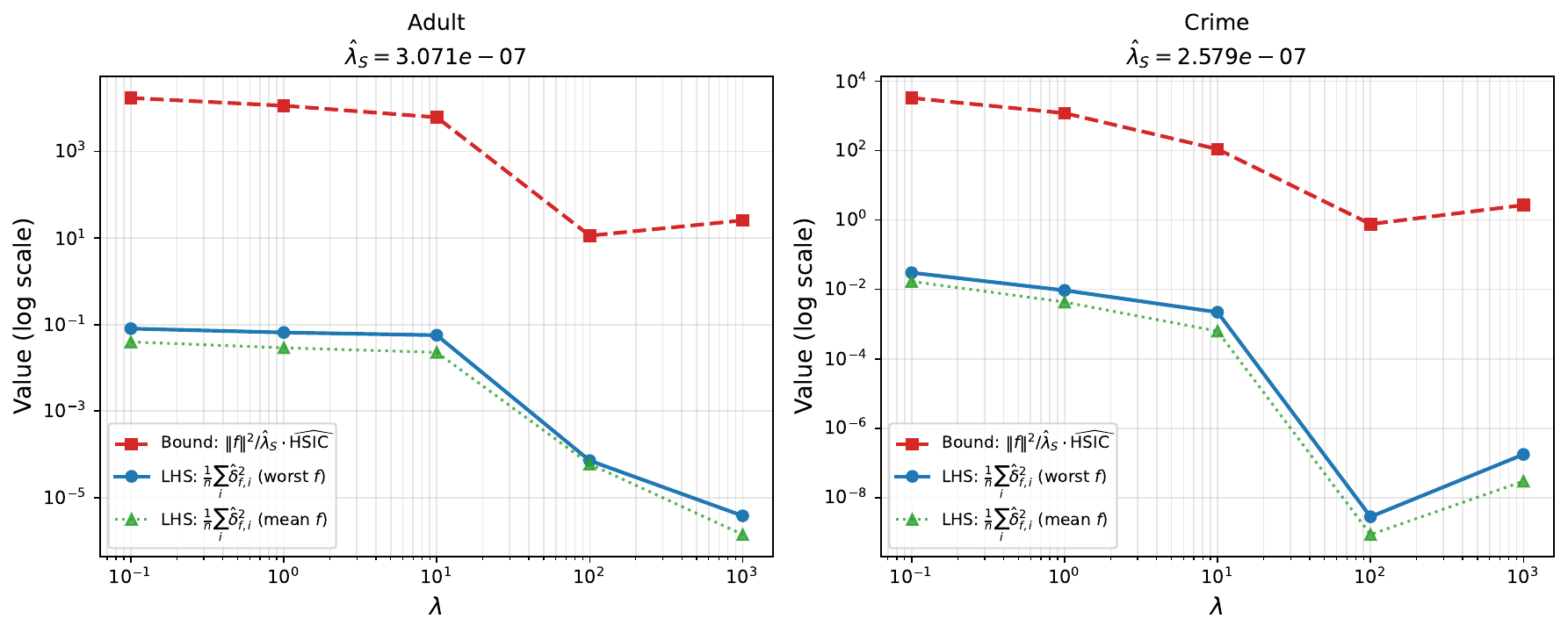}
    \caption{Empirical tightness of the bound from Theorem~\ref{thm:hsic_gdp}. Solid line: empirical LHS $\frac{1}{n}\sum_i \hat\delta_{f, i}^2$. Dashed line: empirical RHS $\|f\|^2 / \hat\lambda_S \cdot \widehat{\hsic}$. Validity: the RHS upper-bounds the LHS at every $\lambda$. The vertical gap can span several orders of magnitude and reflects the looseness of HSIC-based bounds when $\hat\lambda_S$ is small; the bound nonetheless remains finite and non-vacuous, in contrast to the population version, where the constant collapses to zero.}
    \label{fig:bound}
\end{figure}

\paragraph{Bandwidth scaling of $\hat\lambda_S$}
To validate the polynomial-scaling claim in Remark~\ref{rem:lambda2_scaling}, we additionally sweep the bandwidth $\sigma \in \{0.1, 0.3, 1.0, 3.0, 10.0\}$ and record $\hat\lambda_S$ for Adult and Crime at $n = 1500$ (Figure~\ref{fig:lambda2_sigma}). On log--log axes the curves are approximately linear, with negative slopes; the magnitudes are within a factor of two of $d_s = 1$, and the negative sign confirms that the centered Gram matrix on $S$ approaches rank-one as $\sigma$ grows. This empirical polynomial dependence is consistent with Remark~\ref{rem:lambda2_scaling}.

\begin{figure}[ht]
    \centering
    \includegraphics[width=0.7\textwidth]{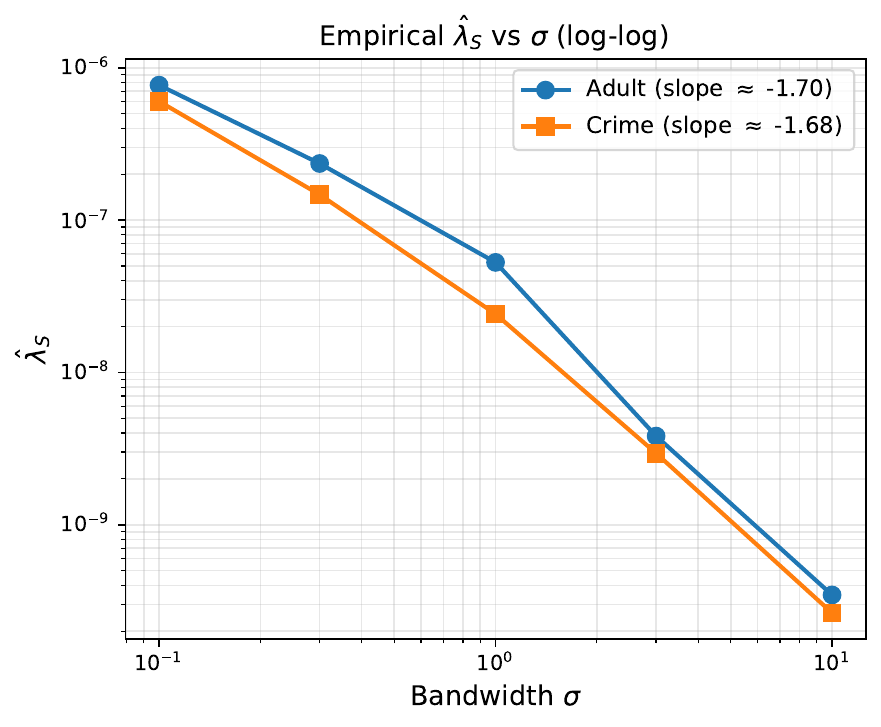}
    \caption{Empirical $\hat\lambda_S$ as a function of bandwidth $\sigma$ on Adult and Crime, $n = 1500$. Log--log axes; the approximately linear scaling (with empirical slopes shown in the legend) supports the polynomial bandwidth dependence noted in Remark~\ref{rem:lambda2_scaling}.}
    \label{fig:lambda2_sigma}
\end{figure}

\subsection{High-Dimensional Sensitive Attributes}
\label{sec:highdim_S}

A practical consequence of computing on the joint distribution rather than on conditional families is that FRHSIC tolerates increasing $\dim(\mathcal{S})$ gracefully, while methods built on the integral framework $\mathcal{I}_d$ rely internally on a kernel-smoothed conditional weighting $\hat w_\gamma(j; i)$ on $\mathcal{S}$ whose statistical accuracy is subject to the curse of dimensionality of nonparametric conditional density estimation.
We illustrate the per-iteration computational consequences on the Adult dataset by constructing $S$ as a $d$-dimensional vector from the continuous Adult features (age, capital gain, capital loss, hours per week, plus a synthetic Gaussian feature for $d = 5$), sweeping $d \in \{1, 2, 3, 4, 5\}$, and comparing the per-epoch training time of FRHSIC and FREM at fixed $\lambda = 10$ and fixed bandwidths ($\sigma_S = 1.0$ for FRHSIC, $\gamma = 0.5$ for FREM); see Figure~\ref{fig:highdim}.

\begin{figure}[ht]
    \centering
    \includegraphics[width=0.7\textwidth]{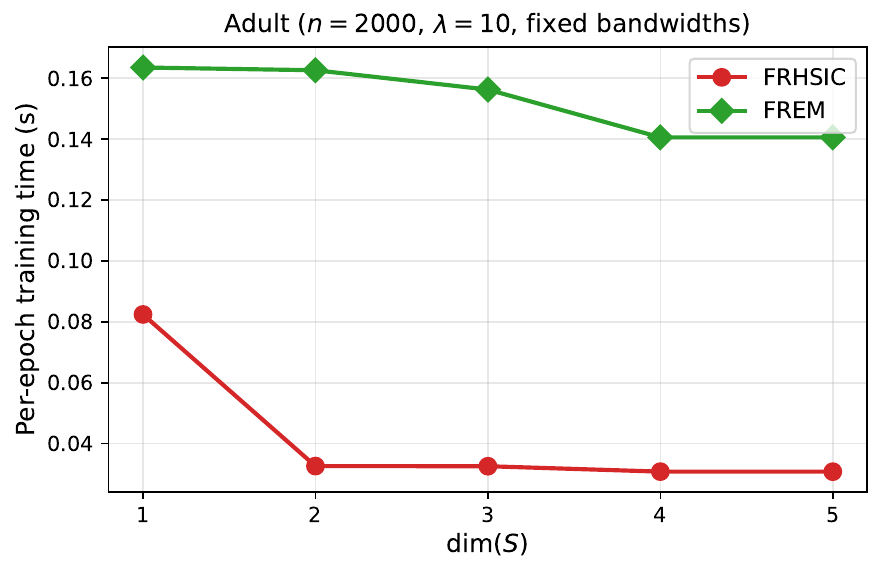}
    \caption{Per-epoch training time on Adult ($n = 2000$) as a function of $\dim(S)$, at fixed $\lambda = 10$ and fixed kernel bandwidths. FRHSIC's HSIC loss is essentially insensitive to $d$ (the dimension enters only through the $n \times n$ kernel evaluation on $S$). FREM is consistently $2$--$3\times$ slower per epoch than FRHSIC across all $d$; the slowdown is dominated by the kernel-smoothed conditional weighting $\hat w_\gamma(j; i)$, which evaluates a $d$-dimensional Gaussian kernel and an anchor-subsampled MMD on every batch.}
    \label{fig:highdim}
\end{figure}

In our wall-clock measurements both curves are roughly flat in $d$ over the range we test, with FREM exhibiting a uniform $\sim 2$--$3\times$ overhead at this scale; the gap grows to roughly $36\times$ at $n = 20{,}000$ (see Section~\ref{sec:runtime}), where FREM's $O(n^3)$ scaling dominates. At $n = 2000$ with an anchor-subsampled MMD ($32$ anchors), the per-batch cost is dominated by the $n \times n$ kernel matrix on $\mathcal{Z}$, which is $d$-independent.
The joint-distribution route is therefore primarily a \emph{statistical} (rather than computational) advantage at moderate $n$: nonparametric conditional weighting on $\mathcal{S}$ inherits the slow rates of $d$-dimensional density estimation, while the HSIC kernel-matrix computation is $d$-agnostic.
The constant gap we observe nonetheless reflects the per-iteration overhead of constructing conditional weights in $\mathcal{I}_d$-style estimators.

\subsection{Ablation Studies}

\paragraph{Kernel choice}
We evaluate FRHSIC with different kernel functions for $k_\mathcal{Z}$ and $k_\mathcal{S}$: Gaussian (RBF), Laplacian, and inverse multiquadric (IMQ).
Table~\ref{tab:ablation_kernel} shows that performance is robust across kernel choices, with the Gaussian kernel performing slightly better overall.
This is consistent with the theoretical requirement that the kernels be characteristic, which all three satisfy.

\begin{table}[t]
\centering
\caption{Ablation over kernel choice (synthetic data, $\lambda = 5$). All kernels achieve comparable performance.}
\label{tab:ablation_kernel}
\begin{tabular}{llcc}
\toprule
$k_\mathcal{Z}$ & $k_\mathcal{S}$ & Accuracy & GDP \\
\midrule
Gaussian & Gaussian & 0.620 & 0.0015 \\
Gaussian & Laplacian & 0.620 & 0.0003 \\
Gaussian & IMQ & 0.620 & 0.0006 \\
Laplacian & Gaussian & 0.619 & 0.0091 \\
Laplacian & Laplacian & 0.620 & 0.0001 \\
IMQ & Gaussian & 0.620 & 0.0002 \\
IMQ & IMQ & 0.620 & 0.0027 \\
\bottomrule
\end{tabular}
\end{table}

\paragraph{Representation dimensionality}
{\sloppy
We vary the representation dimension $d_Z \in \{2, 8, 32, 64\}$ with $\lambda = 5$ (Table~\ref{tab:ablation_dim}).
All dimensions achieve comparable accuracy ($\approx 0.62$), but GDP grows with dimensionality, from $0.0001$ at $d_Z = 2$ to $0.0034$ at $d_Z = 64$.
This confirms that higher-dimensional representations make the HSIC penalty less effective per unit of $\lambda$.
The default $d_Z = 8$ provides a good balance.\par}

\begin{table}[t]
\centering
\caption{Ablation over representation dimensionality (synthetic data, $\lambda = 5$).}
\label{tab:ablation_dim}
\begin{tabular}{ccc}
\toprule
$d_Z$ & Accuracy & GDP \\
\midrule
2  & 0.620 & 0.0001 \\
8  & 0.620 & 0.0001 \\
32 & 0.618 & 0.0025 \\
64 & 0.620 & 0.0034 \\
\bottomrule
\end{tabular}
\end{table}

\subsection{Validation-based regularization selection}
\label{sec:adaptive_lambda}
\label{sec:lambda_exp}

Selecting the regularization strength $\lambda$ is a practical challenge in fair representation learning. We use the HSIC independence test \citep{gretton2005measuring} as a validation-based stopping criterion. The HSIC test statistic under the null hypothesis $H_0: Z \perp S$ follows a weighted sum of chi-squared variables; we use the gamma approximation \citep{gretton2005measuring}, under which $n\cdot\widehat{\mathrm{HSIC}}$ is approximately gamma-distributed with parameters estimated from data.

\paragraph{Algorithm: HSIC-test-based $\lambda$ selection}
\begin{enumerate}
    \item \textbf{Input:} sample $\{(X_i, S_i, Y_i)\}_{i=1}^n$; kernels $k_\cz, k_\cs$; increasing $\lambda$ grid $\Lambda = \{\lambda^{(1)} < \cdots < \lambda^{(K)}\}$; significance level $\alpha_{\mathrm{test}}$ (e.g., $0.05$).
    \item \textbf{For} each $\lambda^{(j)} \in \Lambda$ in increasing order: train FRHSIC at $\lambda^{(j)}$ to obtain $\hat h_{\lambda^{(j)}}$; compute $\widehat{\mathrm{HSIC}}(\hat h_{\lambda^{(j)}}(X), S)$ on a held-out validation split; run the gamma-approximated HSIC independence test at level $\alpha_{\mathrm{test}}$ \citep{gretton2005measuring}; \textbf{if} the test fails to reject $H_0: Z \perp S$, \textbf{return} $\lambda^{(j)}$.
    \item \textbf{Output:} the smallest $\lambda^{(j)}$ at which the test fails to reject independence, or $\lambda^{(K)}$ if no such $\lambda$ is found.
\end{enumerate}
This is a practical validation heuristic based on a familiar independence test, not a principled selection rule. On large validation sets the test can be overpowered, detecting statistically significant but practically negligible dependence. The practical-significance variant stops at the smallest $\lambda^{(j)}$ such that $\widehat{\mathrm{HSIC}} < \epsilon$ or the test fails to reject at level $\alpha_{\mathrm{test}}$, where $\epsilon > 0$ is a user-specified threshold (e.g., $\epsilon = 10^{-3}$), keeping the procedure useful regardless of sample size.

We validate the procedure as follows.
For each dataset, we train FRHSIC over the grid
\[
    \lambda^{(j)} \in \{0.01,\, 0.1,\, 0.5,\, 1,\, 5,\, 10,\, 50,\, 100,\, 500\},
\]
evaluate $\widehat{\mathrm{HSIC}}$ on a held-out validation set ($15\%$ of data), and apply the gamma-approximation independence test at level $\alpha = 0.05$.

Figure~\ref{fig:lambda_selection} shows the results.
On Crime and COMPAS, the test successfully identifies a transition: as $\lambda$ increases, the HSIC test statistic falls below the rejection threshold, selecting $\lambda \approx 5$--$100$ depending on the dataset and random split.
At the selected $\lambda$, the representation achieves low GDP while preserving prediction performance.

On Adult ($n \approx 30{,}000$), the test rejects independence at all $\lambda$ values, including $\lambda = 500$ where GDP is near zero.
This reflects a well-known property of hypothesis testing: with large samples, the test detects statistically significant but \emph{practically negligible} dependence.
In such settings, we recommend supplementing the independence test with a practical significance threshold (e.g., requiring $\widehat{\mathrm{HSIC}} < \epsilon$ for a user-specified $\epsilon$) or using the test on a smaller held-out subsample to reduce power.

\begin{figure}[t]
    \centering
    \includegraphics[width=\textwidth]{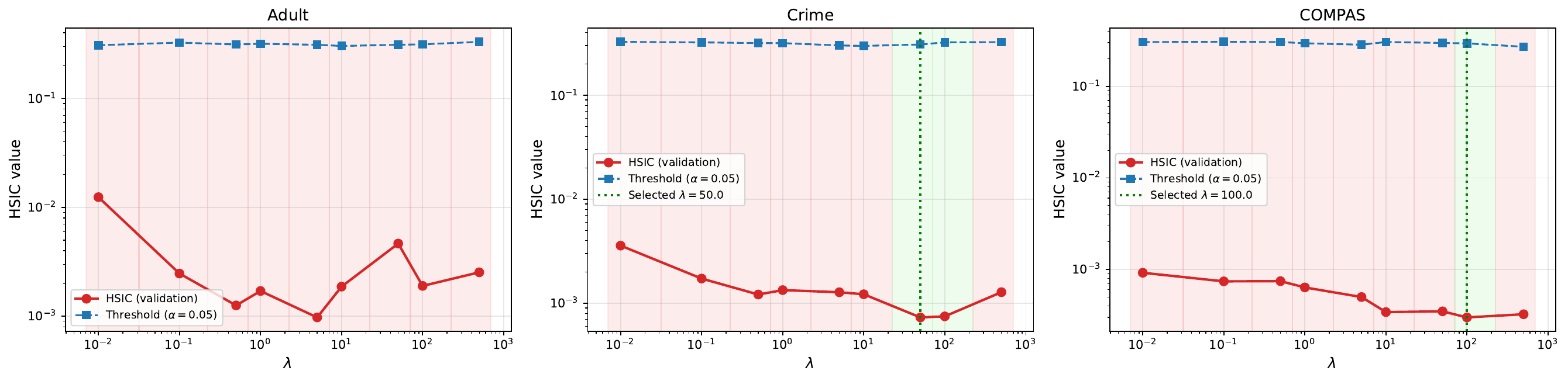}
    \caption{Regularization strength selection via the HSIC independence test (Appendix~\ref{sec:adaptive_lambda}). Red circles: empirical HSIC on the validation set. Blue squares: test threshold at $\alpha = 0.05$. Green shading: $\lambda$ values where the test fails to reject independence. The vertical dashed line marks the selected $\lambda$.}
    \label{fig:lambda_selection}
\end{figure}

\bibliographystyle{plainnat}
\bibliography{refs}

\end{document}